\documentclass[twoside]{article}

\usepackage[preprint]{aistats2026}
\usepackage[round]{natbib}

\usepackage{amsmath,amssymb,amsthm,mathtools,bm}
\usepackage{xspace}
\usepackage{booktabs}
\usepackage{graphicx}
\usepackage{pgf}
\usepackage{enumitem}
\usepackage{hyperref}
\usepackage{xcolor}
\usepackage{float}
\usepackage{algorithm}
\usepackage{algpseudocode}

\newtheorem{theorem}{Theorem}
\newtheorem{proposition}[theorem]{Proposition}
\newtheorem{lemma}[theorem]{Lemma}
\newtheorem{corollary}[theorem]{Corollary}

\theoremstyle{definition}

\newtheorem{remark}[theorem]{Remark}

\newcommand{\Hcal}{\mathcal H}
\newcommand{\E}{\mathbb E}
\newcommand{\Pp}{\mathbb P}
\newcommand{\R}{\mathbb R}
\newcommand{\norm}[1]{\left\lVert #1\right\rVert}
\newcommand{\ip}[2]{\left\langle #1,#2\right\rangle}
\newcommand{\argmin}{\mathop{\mathrm{argmin}}}
\newcommand{\tr}{\mathrm{tr}}

\newcommand{\ind}{\mathbf 1}
\newcommand{\op}{\mathrm{op}}

\begin{document}

\runningtitle{ }

\runningauthor{ }

\twocolumn[

\aistatstitle{HOMER: Huber-of-Means for Efficient and Robust Estimation in Hilbert Spaces}

\aistatsauthor{ Kisung You \And Boram Cho}
\aistatsaddress{ Baruch College \And Yale University} ]
%\aistatsaddress{ Baruch College \And The Graduate Center, City University of New York } ]

\begin{abstract}
Heavy tails weaken high-confidence control for the empirical mean. Geometric median-of-means (MOM) also lacks a threshold that moves toward mean efficiency. We propose \emph{HOMER}, or Huber-of-Means for Efficient and Robust Estimation. HOMER aggregates block means through a radial Huber center. Its canonical and pseudo-Huber forms bound each block score and interpolate between median-like robustness and the empirical mean. We establish a Hilbert-space majority theorem and a MOM-order deviation bound under a finite second moment. Canonical HOMER recovers the sample mean inside its quadratic region. Pseudo-HOMER approaches the sample mean as the threshold grows. It also admits asymptotic linearity and consistent sandwich covariance estimation around the population block-Huber target. Under a finite third moment, fixed finite-dimensional projections support mean inference at the usual parametric rate. This result requires growing block sizes and counts, with block sizes increasing faster. Heavy-tailed simulations show that HOMER remains stable when a minority of block summaries is displaced. On clean Gaussian data, both versions closely approach the empirical mean's efficiency. Finite-block sandwich intervals undercovered, especially for skewed functional data. Further studies show failure when contamination affects most blocks or compromises ordinary within-block means.
\end{abstract}

%----%----%----%----%----%----%----%----%----%----%----%----%----%
\section{INTRODUCTION}

Mean estimation under heavy tails is a basic problem in statistics and machine learning. The empirical mean is natural in Euclidean spaces, functional data spaces, and reproducing kernel Hilbert spaces. Under only a finite second moment, however, its high-confidence behavior is poor. The median-of-means (MoM) principle  addresses this instability by dividing the sample into blocks. It computes each block mean and combines them through a robust median-type rule. In Hilbert and Banach spaces, the standard construction uses the geometric median of the block means \citep{minsker_2015_GeometricMedianRobust,minsker_2024_GeometricMedianApplications}.

MOM separates the problem into two parts. First, each block mean needs only weak accuracy, usually obtained through Chebyshev's inequality. Second, the geometric median converts a majority of weakly accurate blocks into a high-confidence estimator. Together, these provide exponential confidence amplification.

We consider whether other aggregation rules can replace the outer geometric median. The proposed \emph{HOMER}, short for Huber-of-Means for Efficient and Robust Estimation, retains the blockwise architecture of MOM. It replaces the outer geometric median with a radial Huber center of the block means. Given block means $Z_1,\ldots,Z_k$ in a real Hilbert space $\Hcal$, HOMER solves
\begin{equation*}
    \hat\mu_\tau \in \argmin_{\theta\in\Hcal}
    \frac1k\sum_{j=1}^k \rho_\tau\{\norm{Z_j-\theta}\},
\end{equation*}
where $\rho_\tau(t)\sim t^2/2$ near zero and grows linearly in the tails. HOMER is median-like for distant summaries and mean-like for central ones. Figure~\ref{fig:homer-overview} summarizes the construction.

\begin{figure*}[t]
\centering
\includegraphics[width=\textwidth]{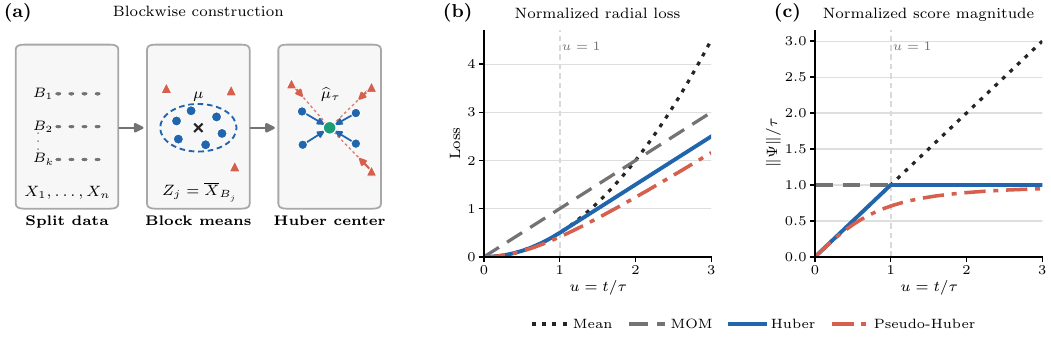}
\caption{\textbf{HOMER at a glance.}
(a) The sample is split into blocks whose means are combined by a radial Huber center. The center balances block scores, and distant contributions are bounded. (b) Normalized radial losses for $u=t/\tau$. (c) Normalized score magnitudes. Rescaling the MOM objective by $\tau$ preserves its minimizer and sets its normalized score to one.}
\label{fig:homer-overview}
\end{figure*}

HOMER addresses robustness, efficiency, and inference jointly. It does not uniformly dominate MOM, which is the median-like endpoint. The Huber threshold provides a controlled path toward mean efficiency and smooth inference.

\paragraph{Contributions.}
Our results address three parts of this tradeoff:
\begin{enumerate}[leftmargin=*,topsep=2pt,itemsep=2pt]
    \item \textbf{Robustness.} We define canonical and pseudo-Huber aggregation for Hilbert-valued block means. The analysis establishes well-posedness, bounded block scores, and a shared two-radius majority theorem. Chebyshev and Hoeffding inequalities then give a MOM-order high-confidence bound under only a finite second moment. This bound does not improve the deviation order of geometric MOM.
    \item \textbf{Efficiency.} We identify the MOM and sample-mean endpoints. Canonical Huber exactly recovers the sample mean in its quadratic basin. Pseudo-Huber converges to the sample mean as its threshold grows. Under a finite third moment, its population sandwich covariance also approaches the sample-mean covariance.
    \item \textbf{Inference.} For smooth pseudo-Huber aggregation, we prove Hilbert-space asymptotic linearity. A plug-in sandwich covariance is also consistent in trace norm. Under a finite third moment, the standardized block-Huber bias is $O(m^{-1/2})$. The sandwich operators converge to Gaussian-block limits. If $m,k\to\infty$ and $k=o(m)$, fixed finite-dimensional projections permit root-$n$ inference for the mean.
\end{enumerate}
The same block summaries support an exact majorization-minimization algorithm. They also support applications to functional means, covariance operators, kernel embeddings, and distributed summaries. The appendix contains proofs, general constants, expanded protocols, and additional examples.

%----%----%----%----%----%----%----%----%----%----%----%----%----%

\section{RELATED WORK}

\paragraph{Robust mean estimation and Huber centers.}
Scalar MOM and its metric-space generalizations combine weak block estimators into high-confidence estimates \citep{nemirovskij_1983_ProblemComplexityMethod,jerrum_1986_RandomGenerationCombinatorial,alon_1999_SpaceComplexityApproximating,hsu_2016_LossMinimizationParameter,minsker_2015_GeometricMedianRobust}. Geometric-median aggregation is the standard Hilbert-space construction \citep{minsker_2024_GeometricMedianApplications}. Scalar Huber merging appears in \citet{minsker_2019_DistributedStatisticalEstimation}. \citet{mathieu_2021_MestimationMedianMeans} calls a scalar block-mean construction HOME (``Huber Of Means Estimator''). Catoni-type truncation provides another bounded-score method for robust means \citep{catoni_2012_ChallengingEmpiricalMean,catoni_2018_DimensionfreePACBayesianBounds}. Adaptive Huber methods examine related threshold tradeoffs in regression \citep{sun_2020_AdaptiveHuberRegression,wang_2021_NewPrincipleTuningFree}.

\paragraph{Applications and inference.}
MONK applies MOM to kernel mean embeddings (KMEs) and maximum mean discrepancy (MMD) \citep{lerasle_2019_MONKOutlierRobustMean}. HOMER adds threshold-controlled efficiency and projected sandwich inference. Geometric medians already have Hilbert-space asymptotics, confidence regions, and robust covariance and principal component analysis (PCA) theory \citep{cardot_2013_EfficientFastEstimation,cardot_2017_OnlineEstimationGeometric,cardot_2017_FastEstimationMedian}. Pseudo-Huber has a specific advantage here. Its regular Hessian is defined everywhere, and the same block summaries yield a direct plug-in sandwich estimate. The MM update resembles Weiszfeld and IRLS algorithms \citep{beck_2015_WeiszfeldsMethodOld}.

%\subsubsection*{Acknowledgements}
%All acknowledgments go at the end of the paper, including thanks to reviewers who gave useful comments, to colleagues who contributed to the ideas, and to funding agencies and corporate sponsors that provided financial support.  To preserve the anonymity, please include acknowledgments \emph{only} in the camera-ready papers. The acknowledgements do not count against the 9-page page limit in the camera-ready.

%----%----%----%----%----%----%----%----%----%----%----%----%----%
\section{HOMER IN HILBERT SPACE}

Let $X_1,\ldots,X_n$ be iid strongly measurable elements of a real separable Hilbert space $(\Hcal,\ip{\cdot}{\cdot},\norm{\cdot})$. Assume $\E\norm X^2<\infty$, set $\mu=\E X$, and write
\begin{equation}\label{eq:finite-second}
    \sigma^2:=\E\norm{X-\mu}^2<\infty.
\end{equation}
For positive integers $k,m$ with $n=km$, split the sample into $k$ disjoint blocks $B_1,\ldots,B_k$ of size $m$. Define
\begin{equation*}
    Z_j=\frac1m\sum_{i\in B_j}X_i,\qquad j=1,\ldots,k.
\end{equation*}
While blocks may be of different cardinality, we use equal block sizes only to simplify notation without loss of generality.

For $\tau>0$, we use two standard losses within the HOMER framework \citep{huber_1964_RobustEstimationLocation,huber_1981_RobustStatistics}. The canonical Huber loss is
\begin{equation}\label{eq:huber-loss}
    h_\tau(t)=
    \begin{cases}
        t^2/2, & 0\le t\le \tau,\\
        \tau t-\tau^2/2, & t>\tau,
    \end{cases}
\end{equation}
and the smooth pseudo-Huber loss is
\begin{equation}\label{eq:pseudo-huber}
    p_\tau(t)=\tau^2\left(\sqrt{1+t^2/\tau^2}-1\right).
\end{equation}
For $\rho_\tau\in\{h_\tau,p_\tau\}$, HOMER is
\begin{equation*}
    \hat\mu_\tau\in\argmin_{\theta\in\Hcal}
    F_\tau(\theta),\qquad
    F_\tau(\theta)=\frac1k\sum_{j=1}^k
    \rho_\tau\{\norm{Z_j-\theta}\}.
\end{equation*}
For the canonical loss, $\hat\mu_\tau$ denotes any Borel-measurable minimizer. A measurable selection exists by parameterizing the finite convex hull with the simplex. The measurable maximum theorem then applies. Every deterministic result below holds for all minimizers, so the selection is immaterial. Both losses are convex, strictly increasing away from zero, and $\tau$-Lipschitz. They also satisfy the scale identity
$\rho_\tau(t)=\tau^2\rho_1(t/\tau)$.

\subsection{Well-posedness and bounded block scores}

Define the radial score maps
\begin{align}
    \Psi^{\mathrm H}_\tau(y)
    &=\begin{cases}
        y, & \norm y\le\tau,\\
        \tau y/\norm y, & \norm y>\tau,
      \end{cases}\label{eq:canonical-score}\\
    \Psi^{\mathrm P}_\tau(y)
    &=\frac{y}{\sqrt{1+\norm y^2/\tau^2}}.\label{eq:pseudo-score}
\end{align}
Both score maps are odd. Therefore, the parameter-minus-data equation below is equivalent to the data-minus-parameter orientation in the inference section.

\begin{proposition}[Existence, convex hull, and score equation]\label{prop:existence}
For either loss in \eqref{eq:huber-loss} and \eqref{eq:pseudo-huber}, $F_\tau$ has a minimizer. Every minimizer belongs to the compact convex hull of $\{Z_1,\ldots,Z_k\}$. The pseudo-Huber objective is strictly convex and therefore has a unique minimizer. Every HOMER solution satisfies
\begin{equation}\label{eq:score}
    \frac1k\sum_{j=1}^k
    \Psi_\tau(\hat\mu_\tau-Z_j)=0,
\end{equation}
where $\Psi_\tau$ is \eqref{eq:canonical-score} or \eqref{eq:pseudo-score}. In both cases,
\begin{equation}\label{eq:bounded-score}
    \norm{\Psi_\tau(y)}\le\tau\qquad\text{for all }y\in\Hcal.
\end{equation}
\end{proposition}

Equation \eqref{eq:bounded-score} gives a bounded \emph{block score}. If the population Hessian is boundedly invertible, the local influence function is also bounded. This second conclusion requires local identification.

\subsection{Endpoint behavior}

The threshold controls how HOMER moves between the geometric median and the sample mean. The following proposition makes these connections precise.

\begin{proposition}[Connections to MOM and the sample mean]\label{prop:endpoints}
Let $\bar Z=k^{-1}\sum_{j=1}^k Z_j$.
\begin{enumerate}[label=(\roman*),leftmargin=*]
    \item As $\tau\downarrow0$, the rescaled objectives $F_\tau/\tau$ converge uniformly on the block means' convex hull to the geometric-median objective. Hence every cluster point of HOMER solutions is a MOM solution.
    \item For the canonical loss, suppose $\max_j\norm{Z_j-\bar Z}\le\tau$. Then $\bar Z$ is a HOMER minimizer. If all inequalities are strict, the minimizer is unique and equals $\bar Z$.
    \item For the pseudo-Huber loss, the unique minimizer converges to $\bar Z$ as $\tau\to\infty$. Its error satisfies $\norm{\hat\mu_\tau-\bar Z}=O(\tau^{-2})$. Appendix~\ref{app:proof-basic} gives an explicit finite-sample bound.
\end{enumerate}
For equal-sized blocks, $\bar Z=n^{-1}\sum_{i=1}^nX_i$.
\end{proposition}

Thus the canonical loss has an exact finite-sample quadratic regime, while pseudo-Huber has a smooth asymptotic quadratic endpoint.
%----%----%----%----%----%----%----%----%----%----%----%----%----%

\section{MAJORITY ROBUSTNESS AND HIGH-CONFIDENCE BOUNDS}

The main Hilbert-space fact is an angle bound. Suppose $D=\norm{\theta-\mu}>r$, $v=(\theta-\mu)/D$, and $z\in B(\mu,r)$. Then
\begin{equation}\label{eq:angle-bound}
    \left\langle v,\frac{\theta-z}{\norm{\theta-z}}\right\rangle
    \ge \sqrt{1-r^2/D^2}.
\end{equation}
The least favorable direction is tangent to the radius-$r$ ball. This inequality compares each good block's outward score with an arbitrary bad block's maximum inward score.

\begin{theorem}[Two-radius majority theorem]\label{thm:deterministic}
Suppose a set $\mathcal G\subset\{1,\ldots,k\}$ satisfies
\begin{equation*}
    |\mathcal G|\ge \frac{5k}{8},
    \qquad
    \norm{Z_j-\mu}\le r\quad\text{for all }j\in\mathcal G.
\end{equation*}
For either loss, suppose $0<\tau\le r$. Then every HOMER solution satisfies
\begin{equation*}
    \norm{\hat\mu_\tau-\mu}\le2r.
\end{equation*}
\end{theorem}

\paragraph{Constants and proof sketch.}
For a HOMER solution farther than $2r$ from $\mu$, \eqref{eq:angle-bound} makes every good block's projected score exceed $\tau\sqrt{3/8}$. Each bad block contributes at worst $-\tau$, so the assumed majority makes the projected score positive and contradicts \eqref{eq:score}. The pair $(5/8,2r)$ is therefore a convenient common choice for both losses, not an optimality claim. Appendix~\ref{app:proof-majority} gives the full proof. Appendix~\ref{app:general-majority} gives sharper loss-specific fraction thresholds and radii, including bounds for every contamination fraction below $1/2$. It also verifies the canonical $\tau\downarrow0$ limit, which equals the Hilbert geometric-median constant of \citet[Lemma~2.1]{minsker_2015_GeometricMedianRobust}.

The probabilistic result follows by applying the majority theorem to independent block means.

\begin{theorem}[High-confidence Hilbert-norm bound]\label{thm:high-prob}
Assume \eqref{eq:finite-second} and $\sigma>0$. For either loss, take
\[
    0<\tau\le 2\frac{\sigma}{\sqrt m}.
\]
Then
\begin{equation*}
    \Pp\left\{
    \norm{\hat\mu_\tau-\mu}>4\frac{\sigma}{\sqrt m}
    \right\}
    \le e^{-k/32}.
\end{equation*}
For $0<\delta<1$, take $k=\lceil32\log(1/\delta)\rceil$. Suppose $k\le n$ and $k$ divides $n$. Then, with probability at least $1-\delta$,
\begin{equation}\label{eq:delta-bound}
    \norm{\hat\mu_\tau-\mu}
    \le4\sigma\sqrt{\frac{k}{n}}
    \lesssim \sigma\sqrt{\frac{\log(e/\delta)}{n}}.
\end{equation}
When $\sigma=0$, all observations equal $\mu$ almost surely and the conclusion is trivial.
\end{theorem}
If $k\nmid n$, one may discard at most $k-1$ observations. Alternatively, block sizes may differ by at most one. Only the displayed constants change, according to the smallest block size. The deterministic event controls every canonical-Huber minimizer. Therefore, any measurable selection obeys the same probability bound.

In $\R^d$, $\sigma^2=\tr(\Sigma)$. Thus \eqref{eq:delta-bound} gives a trace-type Hilbert-norm result similar to geometric MOM. It is not anisotropically optimal. General bounds tolerate any corrupted-summary fraction below $1/2$. The two-radius theorem uses $3/8$. Neither result implies robustness to one adversarial point in every block.
%----%----%----%----%----%----%----%----%----%----%----%----%----%

\section{APPLICATIONS IN HILBERT SPACES}\label{sec:applications}

Many targets are Hilbert-valued means. For each rate below, take $0<\delta<1$, $k=\lceil32\log(1/\delta)\rceil\le n$, $k\mid n$, and equal blocks. Set $0<\tau\le2s/\sqrt m$, where $s^2$ is the relevant Hilbert variance. If $s=0$, the estimator is exact for any $\tau>0$. Each bound holds with probability at least $1-\delta$. Unequal blocks use the smallest-block correction.

\paragraph{Functional mean curves.}
For strongly measurable $X\in L^2[0,1]$, $\mu=\E X$ is the mean trajectory as an $L^2$ equivalence class. Suitable representatives satisfy $\mu(t)=\E X(t)$ almost everywhere. Here $s^2=\E\norm{X-\mu}_{L^2}^2$, and the theorem gives
\[
    \norm{\hat\mu_\tau-\mu}_{L^2}
    \lesssim
    \{\E\norm{X-\mu}_{L^2}^2\}^{1/2}
    \sqrt{\frac{\log(e/\delta)}{n}}.
\]
For discretized curves, this is a bound in the chosen grid norm. A continuous $L^2$ claim also requires discretization control. Daily demand curves in the additional studies may be dependent. The iid theory therefore serves only as a benchmark for that analysis. Raw-observation Huber centers are studied by \citet{sinova_2018_MestimatorsLocationFunctional}.

\paragraph{Covariance operators and robust PCA.}
For the direct corollary, let $X\in\Hcal$ have known mean $\mu$, and let
\[
    Y=(X-\mu)\otimes (X-\mu)
\]
be the rank-one covariance operator. With sample splitting, condition on an independent center $\tilde\mu$. The estimation-subsample target is $\Sigma+(\tilde\mu-\mu)\otimes(\tilde\mu-\mu)$. Thus error to $\Sigma$ adds $\norm{\tilde\mu-\mu}^2$ to the conditional HOMER term. Here $n$ denotes the covariance-estimation subsample size. The Hilbert-Schmidt class $\mathcal S_2(\Hcal)$ is a Hilbert space, and $\Sigma=\E Y$ under a finite second moment. If
\[
    \nu^2=\E\norm{Y-\Sigma}_{\mathcal S_2}^2<\infty,
\]
then take $s=\nu$. Applying HOMER to $Y_i$ yields $\hat\Sigma_\tau$ satisfying
\begin{equation*}
    \norm{\hat\Sigma_\tau-\Sigma}_{\mathcal S_2}
    \lesssim
    \nu\sqrt{\frac{\log(e/\delta)}{n}}.
\end{equation*}
A finite fourth moment of $X$ is sufficient. HOMER lies in the convex hull of positive semidefinite block covariance operators, so $\hat\Sigma_\tau$ is positive semidefinite. Suppose the leading $r$-dimensional eigenspace of $\Sigma$ has gap $\Delta$, and the perturbation is below a fixed fraction of $\Delta$. Then we have 
\begin{align*}
    \norm{\hat P_r-P_r}_{\op}
    &\lesssim
    \frac{\norm{\hat\Sigma_\tau-\Sigma}_{\op}}{\Delta},\\
    \norm{\hat P_r-P_r}_{\mathcal S_2}
    &\lesssim
    \sqrt r\,\frac{\norm{\hat\Sigma_\tau-\Sigma}_{\op}}{\Delta},
\end{align*}
by the Davis-Kahan theorem \citep{davis_1970_RotationEigenvectorsPerturbation}. Since $\norm{A}_{\op}\le\norm{A}_{\mathcal S_2}$, the HOMER Hilbert-Schmidt bound controls both displays. Here $P_r$ and $\hat P_r$ denote the population and estimated leading eigenspace projectors. This result is the HOMER analogue of robust PCA based on geometric aggregation or median covariation \citep{minsker_2015_GeometricMedianRobust,cardot_2017_FastEstimationMedian}.

\paragraph{Kernel mean embeddings.}
Let $W\sim P$, and let $K$ be a measurable positive definite kernel with real separable RKHS $\mathcal K$. Assume $X=K(W,\cdot)$ is strongly measurable. Its target
\[
    \mu_P=\E K(W,\cdot)
\]
is the KME of distribution $P$. Suppose $K(w,w)\le \kappa^2$ for all $w$. Then $s_K^2:=\E\norm{X-\mu_P}_{\mathcal K}^2\le \kappa^2$. Taking $s=s_K$, HOMER gives
\begin{equation}\label{eq:kme-app-bound}
    \norm{\hat\mu_{P,\tau}-\mu_P}_{\mathcal K}
    \lesssim
    \kappa\sqrt{\frac{\log(e/\delta)}{n}}.
\end{equation}
The standard RKHS inequality is
$|\E_P f-\ip{f}{\hat\mu_{P,\tau}}_{\mathcal K}|\le \norm{f}_{\mathcal K}\norm{\hat\mu_{P,\tau}-\mu_P}_{\mathcal K}$.
Together with \eqref{eq:kme-app-bound}, it controls every function in the RKHS unit ball. HOMER kernel means can also replace empirical kernel means in robust maximum mean discrepancy calculations \citep{gretton_2012_KernelTwoSampleTest,muandet_2017_KernelMeanEmbedding}. MONK already develops MOM estimators and tests for KME and MMD \citep{lerasle_2019_MONKOutlierRobustMean}. HOMER contributes a smooth threshold path and projected sandwich theory, rather than the first robust kernel embeddings.

\paragraph{Distributed gradients and model summaries.}
At a fixed parameter, stochastic gradients and influence evaluations are often Hilbert-valued means. HOMER combines worker summaries with bounded block scores. Local influence is bounded when the population Hessian is boundedly invertible. Independent honest summaries must share a target and the stated second-moment bound. Arbitrary summaries and inaccurate honest summaries must together stay below the theorem's bad-block threshold. This is pointwise aggregation, not an iterative or Byzantine optimization theorem \citep{blanchard_2017_MachineLearningAdversaries,yin_2018_ByzantineRobustDistributedLearning}.

%----%----%----%----%----%----%----%----%----%----%----%----%----%

\section{EFFICIENCY AND INFERENCE}

The finite-sample theorem establishes robustness for the smooth pseudo-Huber loss used below. We next describe its efficiency and inferential properties.

\subsection{Efficiency path}

Proposition~\ref{prop:endpoints} gives the finite-sample endpoints. The canonical loss returns the full sample mean when all residuals lie strictly inside its quadratic region. The pseudo-Huber solution converges to the sample mean as $\tau\to\infty$. Appendix~\ref{app:large-threshold} gives the corresponding population covariance result. Consider a centered block variable with a finite third moment. Its pseudo-Huber population center is $O(\lambda^{-2})$. Its sandwich covariance converges to the ordinary mean's covariance as the standardized threshold $\lambda\to\infty$.

For robust high-confidence estimation, Theorem~\ref{thm:high-prob} suggests setting $\tau$ on the block-error scale $\sigma/\sqrt m$. Larger thresholds retain more quadratic information but weaken clipping. Thus $\tau$ sets the tradeoff between robustness and efficiency.

\subsection{A Hilbert-space sandwich theorem}

Use pseudo-Huber aggregation and write $\tau_m=\lambda/\sqrt m$ for a fixed deterministic standardized threshold $\lambda>0$. Define
\[
    Y_{m,j}=\sqrt m(Z_j-\mu),
    \qquad
    \hat u=\sqrt m(\hat\mu_{\tau_m}-\mu).
\]
The scale identity turns the HOMER objective into
\[
    u\mapsto \frac1k\sum_{j=1}^k
    p_\lambda\{\norm{Y_{m,j}-u}\}.
\]
Let $u_{m,\lambda}$ be the unique minimizer of the population objective and define
\begin{equation*}
    \theta_{m,\lambda}=\mu+\frac{u_{m,\lambda}}{\sqrt m}.
\end{equation*}
For $y\in\Hcal$, set
\begin{align*}
    \phi_\lambda(y)
    &=\frac{y}{\sqrt{1+\norm y^2/\lambda^2}},\\
    H_\lambda(y)
    &=D\phi_\lambda(y),
\end{align*}
where $D$ is the Fr\'echet derivative. At the population center, define the expected Hessian and score covariance
\begin{align*}
    A_{m,\lambda}
    &=\E H_\lambda(Y_{m,1}-u_{m,\lambda}),\\
    B_{m,\lambda}
    &=\E\big[\phi_\lambda(Y_{m,1}-u_{m,\lambda})
       \otimes\phi_\lambda(Y_{m,1}-u_{m,\lambda})\big].
\end{align*}
The bounds $\norm{\phi_\lambda}\le\lambda$ and $\norm{H_\lambda}_{\op}\le1$ justify differentiation under the expectation. Thus the population score vanishes at $u_{m,\lambda}$. Appendix~\ref{app:inference} proves that $A_{m,\lambda}$ is self-adjoint and boundedly invertible.

\begin{theorem}[Fixed-block asymptotic linearity]\label{thm:fixed-m-inference}
Fix $m$ and $\lambda>0$, and let $k\to\infty$. Under \eqref{eq:finite-second}, with $n=mk$,
\begin{align*}
    &\sqrt n\,(\hat\mu_{\lambda/\sqrt m}-\theta_{m,\lambda})\\
    &\quad=A_{m,\lambda}^{-1}\frac1{\sqrt k}
      \sum_{j=1}^k\phi_\lambda(Y_{m,j}-u_{m,\lambda})+o_p(1).
\end{align*}
Consequently,
\begin{align*}
    \sqrt n\,(\hat\mu_{\lambda/\sqrt m}-\theta_{m,\lambda})
    &\Rightarrow \mathcal N_{\Hcal}(0,V_{m,\lambda}),\\[-2mm]
    V_{m,\lambda}
    &=A_{m,\lambda}^{-1}B_{m,\lambda}A_{m,\lambda}^{-1}.
\end{align*}
If $X-\mu$ is centrally symmetric, then $u_{m,\lambda}=0$ and $\theta_{m,\lambda}=\mu$ for every $m$ and $\lambda$.
\end{theorem}

The same $k$ block summaries provide a consistent sandwich estimate. Put $\hat Y_j=\sqrt m(Z_j-\hat\mu_{\lambda/\sqrt m})$ and define
\begin{align*}
    \hat A_{m,\lambda}
    &=\frac1k\sum_{j=1}^kH_\lambda(\hat Y_j),\\
    \hat B_{m,\lambda}
    &=\frac1k\sum_{j=1}^k
      \phi_\lambda(\hat Y_j)\otimes\phi_\lambda(\hat Y_j),\\
    \hat V_{m,\lambda}
    &=\hat A_{m,\lambda}^{-1}\hat B_{m,\lambda}\hat A_{m,\lambda}^{-1}.
\end{align*}
For trace-class $T$, write $\norm{T}_1=\tr\{(T^*T)^{1/2}\}$. Appendix~\ref{app:inference} proves $\norm{\hat V_{m,\lambda}-V_{m,\lambda}}_1\to_p0$. A fixed deterministic contrast permits Wald inference around $\theta_{m,\lambda}$ when its covariance is nonsingular.

The target shift follows from model primitives, not an added assumption. Let $\Gamma=\E[(X-\mu)\otimes(X-\mu)]$ and $G\sim\mathcal N_{\Hcal}(0,\Gamma)$. Define
\begin{align*}
    A_\lambda^{G}&=\E H_\lambda(G),\\
    B_\lambda^{G}&=\E\{\phi_\lambda(G)\otimes\phi_\lambda(G)\},\\
    V_\lambda^{G}&=(A_\lambda^{G})^{-1}B_\lambda^{G}(A_\lambda^{G})^{-1}.
\end{align*}

\begin{proposition}[Primitive block-bias and operator limits]\label{prop:primitive-bias}
If $\beta_3:=\E\norm{X-\mu}^3<\infty$, then for every fixed $\lambda>0$,
\begin{align*}
    \norm{u_{m,\lambda}}&=O(m^{-1/2}),\\
    \norm{A_{m,\lambda}-A_\lambda^{G}}_{\op}&\to0,\\
    \norm{B_{m,\lambda}-B_\lambda^{G}}_1&\to0.
\end{align*}
The constants in the bias bound depend only on $\lambda$, $\sigma$, and $\beta_3$.
\end{proposition}

\begin{corollary}[Projected inference for the mean]\label{cor:triangular}
Assume Proposition~\ref{prop:primitive-bias}. Let $m,k\to\infty$ with $k=o(m)$. For each fixed $q$ and fixed deterministic bounded map $L:\Hcal\to\R^q$,
\[
    \sqrt n\,L(\hat\mu_{\lambda/\sqrt m}-\mu)
    \Rightarrow
    N_q\!\left(0,L V_\lambda^{G}L^*\right),
\]
where $L^*$ is the Hilbert adjoint. Moreover,
$L\hat V_{m,\lambda}L^*\to_p L V_\lambda^{G}L^*$ in matrix norm.
\end{corollary}

Thus $k\asymp n^\gamma$ supports asymmetric mean inference for $0<\gamma<1/2$. Under central symmetry, $k=o(m)$ is unnecessary for centering. The Gaussian-block limit still requires $k,m\to\infty$. The bias constant deteriorates as $\lambda\downarrow0$ because mean-target control weakens near MOM. Geometric medians also have inference theory \citep{cardot_2013_EfficientFastEstimation,cardot_2017_OnlineEstimationGeometric}. HOMER supplies a smooth blockwise Hessian and direct sandwich estimate.

%----%----%----%----%----%----%----%----%----%----%----%----%----%

\section{ALGORITHM}

For either loss, define the continuous weight function
\[
 w_\tau(r)=
 \begin{cases}
  1\wedge(\tau/r), & \rho_\tau=h_\tau,\\
  (1+r^2/\tau^2)^{-1/2}, & \rho_\tau=p_\tau,
 \end{cases}
 \qquad w_\tau(0)=1.
\]
The majorization-minimization update is
\begin{equation}\label{eq:mm-update}
    \theta^{(t+1)}=
    \frac{\sum_{j=1}^k w_\tau(r_j^{(t)})Z_j}
         {\sum_{j=1}^k w_\tau(r_j^{(t)})},
    \qquad r_j^{(t)}=\norm{Z_j-\theta^{(t)}}.
\end{equation}

\begin{algorithm}[t]
\caption{HOMER by majorization-minimization}
\label{alg:homer-mm}
\begin{algorithmic}[1]
\Require Block means $Z_1,\ldots,Z_k$, threshold $\tau>0$, loss $\rho_\tau$, tolerance $\varepsilon$, maximum iterations $T\ge1$.
\State Initialize $\theta^{(0)}\gets k^{-1}\sum_{j=1}^k Z_j$.
\For{$t=0,\ldots,T-1$}
    \For{$j=1,\ldots,k$}
        \State $r_j^{(t)}\gets\norm{Z_j-\theta^{(t)}}$; $w_j^{(t)}\gets w_\tau(r_j^{(t)})$.
    \EndFor
    \State Update $\theta^{(t+1)}$ by \eqref{eq:mm-update}.
    \If{$\norm{\theta^{(t+1)}-\theta^{(t)}}\le\varepsilon\max\{1,\norm{\theta^{(t)}}\}$}
        \State \textbf{break}
    \EndIf
\EndFor
\State \Return the final iterate $\theta^{(t+1)}$.
\end{algorithmic}
\end{algorithm}

The update is an exact Huber-weighted Weiszfeld/IRLS majorization step \citep{beck_2015_WeiszfeldsMethodOld,cardot_2013_EfficientFastEstimation}. Therefore, $F_\tau(\theta^{(t+1)})\le F_\tau(\theta^{(t)})$ without a line search. For pseudo-Huber, strict convexity gives a unique target. Moreover, the iterates converge to it, which is shown in Appendix~\ref{app:algorithm}. The computation uses only the $k$ block means. In an RKHS, their $k\times k$ Gram matrix provides the distances and updates.

%----%----%----%----%----%----%----%----%----%----%----%----%----%

\section{EXPERIMENTS}\label{sec:experiments}

The experiments examine resistance to displaced block summaries, recovery of mean efficiency, and projected uncertainty at finite $k$. The first simulated example studies these properties for a finite-dimensional Hilbert mean. The second simulated example uses function-valued observations. A real-data example treats subjects as sampling units in a trajectory analysis. Additional experiments for different types of objects are deferred to Appendix~\ref{app:experiments}.

\paragraph{Protocol and reporting scope.}
Pseudo-Huber HOMER is primary, while canonical Huber displays the exact quadratic basin. Baselines include the empirical mean and geometric MOM, with task-specific raw-Huber, coordinatewise-MOM, and radial-truncation comparisons. Unless inference fixes $\tau=\lambda/\sqrt m$, the threshold has the form
\[
    \hat\tau=c\cdot \mathrm{median}_{1\le j\le k}\norm{Z_j-\tilde\mu},
\]
where $\tilde\mu$ is a preliminary MOM center. Threshold grids depend on the study. The first simulated example uses $c\in\{0.5,2,8\}$ for its displayed efficiency path. Broader sensitivity paths may use $c\in\{1,2,4,8\}$. Other displayed point comparisons use $c=2$. Methods within a simulation replicate share data, blocks, and perturbations. The displayed synthetic results use a reduced grid with 500 replications. The two simulated examples use 20,000 independent draws to approximate finite-block targets. Horizontal axes report the realized fraction $\lfloor\epsilon k\rfloor/k$ whenever whole blocks are selected. Data-driven thresholds and fixed-$k$ coverage are empirical diagnostics outside the adaptive and triangular-array theories.

\subsection{Robustness, efficiency, and inference.}
The first simulated example combines the three aims in a controlled Hilbert-mean setting. Student-$t_3$ samples test high-quantile stability under complete-block shifts. Clean Gaussian samples trace the threshold path. Gaussian and skewed samples assess projected sandwich intervals for the mean and finite-block pseudo-Huber target.

\begin{figure*}[t]
\centering
\resizebox{\textwidth}{!}{\input{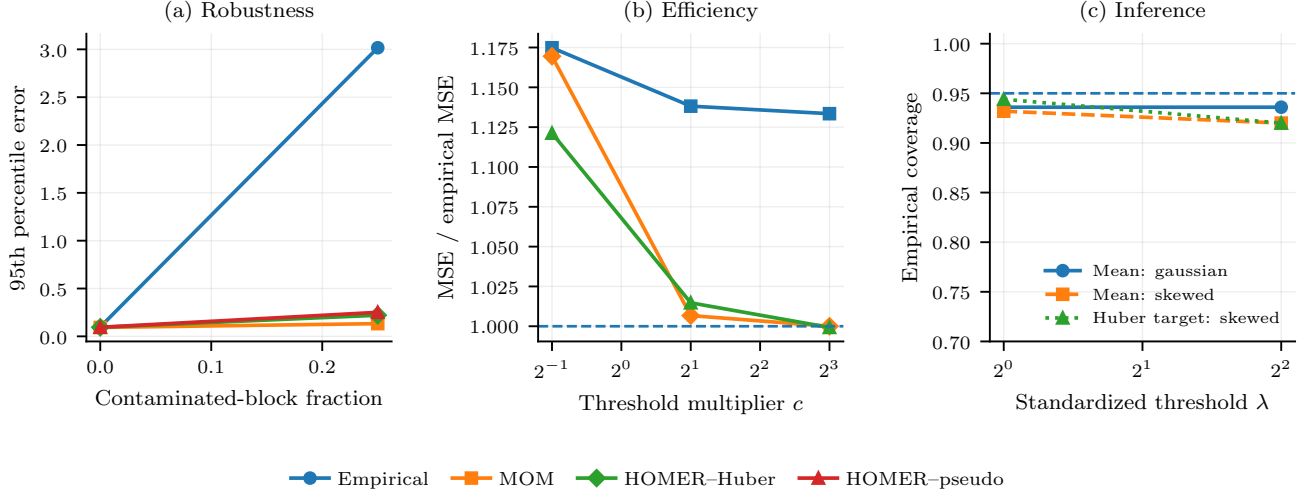}}
\caption{\textbf{Finite-dimensional Hilbert mean: robustness, efficiency, and inference.}
(a) Student-$t_3$ data have $n=256$, $d=40$, $k=16$, and $c=2$. Error is the 95th percentile over 500 replications. (b) Clean Gaussian MSE is relative to the empirical mean for $c\in\{0.5,2,8\}$. (c) Empirical coverage uses a fixed projection with $n=256$, $d=20$, and $k=16$. The horizontal line is 0.95.}
\label{fig:hilbert-mean-overview}
\end{figure*}

With one quarter of the block summaries displaced, the empirical mean's 95th-percentile error increased from $0.098$ to $3.017$. Geometric MOM, canonical HOMER, and pseudo-HOMER yielded $0.133$, $0.221$, and $0.251$ under the same stress. On clean Gaussian data, canonical and pseudo-HOMER were within $1.5\%$ of empirical-mean MSE at $c=2$. Both reached the mean endpoint at $c=8$. Geometric MOM remained between $13\%$ and $17\%$ above empirical MSE across this grid. Gaussian mean coverage was $0.936$ at both $\lambda=1$ and $4$. Under skewness, mean coverage was $0.932$ and $0.920$, while coverage of the simulated block-Huber target was $0.944$ and $0.920$. Bounded-score methods remained stable under this stress. The finite-$k$ intervals, however, remained mildly below nominal coverage.

\subsection{Functional means and projected inference}
Fourier coefficients represent random curves in an orthonormal basis, so coefficient error is exactly truncated $L^2$ error. The inferential targets are three continuous linear functionals. They comprise an interval average, a Gaussian-smoothed evaluation, and an evening-minus-morning contrast. These functionals are continuous on $L^2$. Literal point evaluation is not.

Under Student-$t_3$ sampling, one-quarter block contamination raised the empirical 95th-percentile $L^2$ error from $0.112$ to $2.524$. The stressed errors were $0.132$ for MOM, $0.191$ for canonical HOMER, and $0.211$ for pseudo-HOMER. For the evening-minus-morning contrast, Gaussian mean coverage was $0.924$ at $\lambda=1$ and $0.906$ at $\lambda=4$. Under skewness, mean coverage was $0.882$ and $0.854$, with nearly identical finite-block-target coverage. In this experiment, block aggregation limited the $L^2$ error increase, but the eight-block sandwich intervals were undercalibrated.

\subsection{Wearable trajectory data}
The UCI Human Activity Recognition (HAR) application \citep{jorgereyes-ortiz_2013_HumanActivityRecognition} averages overlapping windows within subject and activity before estimation. This leaves 30 subject-level sampling units rather than thousands of overlapping windows. Subject-row and sensor-channel perturbations assess stability. Held-out-subject nearest-template classification provides only a downstream diagnostic.

\begin{figure*}[t]
\centering
\resizebox{\textwidth}{!}{\input{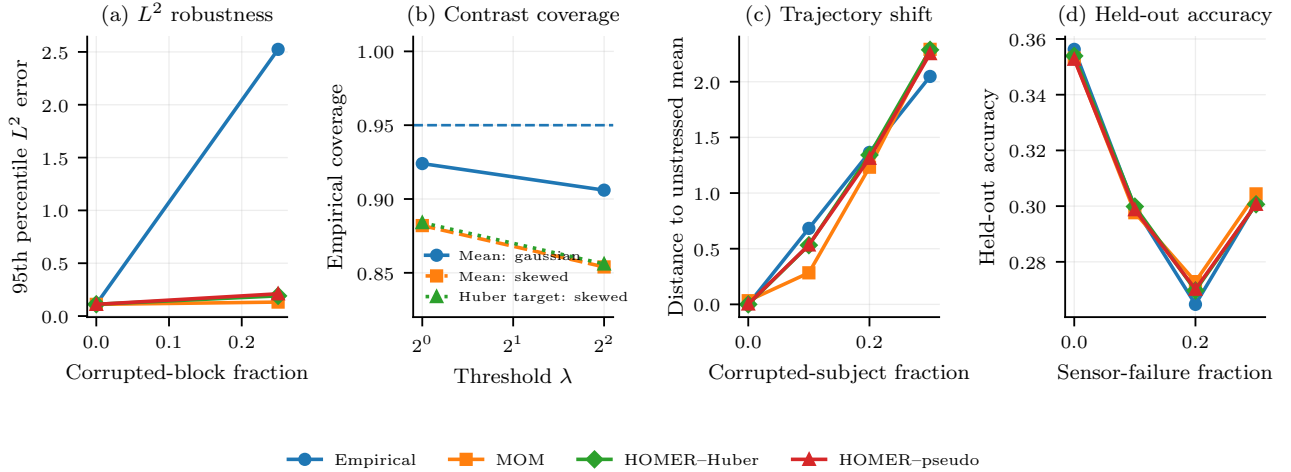}}
\caption{\textbf{Functional and subject-level applications.}
Panels (a) and (b) show reduced-grid functional robustness and evening-minus-morning coverage. Each panel uses 500 replications. Panel (c) shows distance from the unstressed empirical WALKING trajectory under subject spikes at $c=2$. Panel (d) gives held-out nearest-template accuracy under sensor failure.}
\label{fig:main-applications}
\end{figure*}

At $c=2$ with 10\% subject-spike contamination, ordinary averaging moved the WALKING trajectory $0.683$ from its unstressed empirical reference. The corresponding distance was $0.283$ for MOM and about $0.53$ for each HOMER variant. The separation narrowed at 20\% and disappeared at 30\%, as corruption spread across most ordinary-mean blocks. Under sensor failures, nearest-template accuracies were similar across estimators at every stress level. These classification results do not indicate predictive superiority. 

%----%----%----%----%----%----%----%----%----%----%----%----%----%
\section{DISCUSSION}

HOMER should be viewed as a thresholded alternative to geometric MOM, not a uniformly stronger estimator. Its finite-sample bound depends on total Hilbert variance and is not anisotropically optimal. The experiments also show no uniform advantage, and geometric MOM can be more stable under severe block contamination. Related geometric-median methods already support asymptotic inference, confidence regions, and robust PCA \citep{cardot_2013_EfficientFastEstimation,cardot_2017_OnlineEstimationGeometric,cardot_2017_FastEstimationMedian}. Directional guarantees would better describe high-dimensional covariance and kernel applications. Comparisons should also vary how contamination reaches blocks.

The inferential theory has narrower scope than the estimation result. At fixed $m$, sandwich inference targets $\theta_{m,\lambda}$ rather than the mean unless symmetry removes the block bias. For asymmetric distributions, mean inference requires a finite third moment, fixed finite-dimensional projections, and $k=o(m)$ as $k$ and $m$ grow. The inferential experiments instead use eight or sixteen blocks and show undercoverage, especially for skewed functional data. Finite-sample calibration should address block-target bias and the small numbers of summaries used in practice. Growing-dimensional or data-dependent projections would require uniform covariance and bias control. Canonical-Huber inference may be possible when the block-mean distribution assigns no mass to the threshold sphere. Its discontinuous empirical Hessian still requires an argument beyond the smooth pseudo-Huber analysis.

The robustness guarantee concerns block summaries, not arbitrary raw-point contamination. Dispersed contamination can compromise most ordinary block means even when only a small fraction of observations is corrupted. Replacing ordinary block means with robust summaries could address this case. This change would require separate bias and efficiency analyses in Hilbert space. The deviation theorem also uses an oracle threshold based on $\sigma$, while the median-distance rule used in experiments lacks theoretical coverage. The Lepski-type grid in Appendix~\ref{app:tuning} offers one route toward adaptive guarantees. Joint selection of $k$ and $\tau$ remains open because both choices affect robustness, bias, and interval calibration. Finally, the financial and demand studies involve dependence outside the iid framework. Extensions to dependent blocks and time-varying targets are needed before those applications support formal uncertainty statements.

% If you use BibTeX in apalike style, activate the following line:
\bibliographystyle{apalike}
\bibliography{references}

%%%%%%%%%%%%%%%%%%%%%%%%%%%%%%%%%%%%%%%%%%%%%%%%%%%%%%%%%%%%

%\clearpage
\appendix
%\thispagestyle{empty}

% Supplementary material: To improve readability, you must use a single-column format for the supplementary material.
%\onecolumn
%\aistatstitle{Supplementary M}

\section{PROOFS}\label{app:proofs}

\subsection{Basic properties}\label{app:proof-basic}

\begin{proof}[Proof of Proposition~\ref{prop:existence}]
Let $C=\operatorname{conv}\{Z_1,\ldots,Z_k\}$. It is a compact subset of the finite-dimensional span of the block means. If $\Pi_C\theta$ denotes the metric projection onto $C$, the Hilbert projection inequality gives
\[
    \norm{z-\Pi_C\theta}^2
    \le \norm{z-\theta}^2-\norm{\theta-\Pi_C\theta}^2,
    \qquad z\in C.
\]
When $\theta\notin C$, projection strictly reduces every distance to a block mean. Both radial losses are strictly increasing. Therefore, $F_\tau(\Pi_C\theta)<F_\tau(\theta)$. Every minimizer lies in $C$, and continuity on compact $C$ gives existence. Write each $\theta\in C$ as $\sum_j a_jZ_j$, where $a$ belongs to the compact simplex. This parameterization gives a measurable compact-valued argmin problem. The measurable maximum theorem supplies a Borel selection as a function of $(Z_1,\ldots,Z_k)$.

Both radial objectives are convex and Fr\'echet differentiable. Their gradients are the score maps in \eqref{eq:canonical-score} and \eqref{eq:pseudo-score}, including at the origin. The first-order condition for a convex differentiable objective gives \eqref{eq:score}. The displayed formulas give \eqref{eq:bounded-score}. Finally, $y\mapsto p_\tau(\norm y)$ has an everywhere positive-definite Hessian and is strictly convex. Averaging preserves strict convexity.
\end{proof}

\begin{proof}[Proof of Proposition~\ref{prop:endpoints}]
Let $C=\operatorname{conv}\{Z_1,\ldots,Z_k\}$ and
$M_C=\max_{\theta\in C,\,1\le j\le k}\norm{\theta-Z_j}<\infty$.
For canonical Huber,
\[
    \sup_{0\le t\le M_C}\left|\frac{h_\tau(t)}{\tau}-t\right|\le\frac\tau2,
\]
and for pseudo-Huber,
\[
    \frac{p_\tau(t)}{\tau}=\sqrt{t^2+\tau^2}-\tau,
    \qquad
    \sup_{0\le t\le M_C}\left|\frac{p_\tau(t)}{\tau}-t\right|\le\tau.
\]
The residuals are uniformly bounded on compact $C$. Therefore, both rescaled objectives converge uniformly to the geometric-median objective. Argmin continuity proves part (i).

For part (ii), suppose every residual at $\bar Z$ is at most $\tau$. Then
\[
    \nabla F_\tau(\bar Z)=\frac1k\sum_{j=1}^k(\bar Z-Z_j)=0.
\]
Convexity makes $\bar Z$ a global minimizer. If every residual is strictly below $\tau$, the objective is locally equal to a strictly convex quadratic objective. A second global minimizer would make the segment between both minimizers flat. This would contradict local strict convexity. Hence the minimizer is unique.

For part (iii), the pseudo-Huber minimizer lies in compact $C$. On the bounded residual range generated by $C$,
\[
    p_\tau(t)\longrightarrow t^2/2
\]
uniformly as $\tau\to\infty$. The limiting quadratic objective has unique minimizer $\bar Z$. Uniform argmin convergence therefore proves convergence. For the rate, let $D_C=\max_{i,j}\norm{Z_i-Z_j}$. Every residual generated by $C$ is at most $D_C$. The pseudo-Huber Hessian is bounded below by $c_\tau I$, where $c_\tau=(1+D_C^2/\tau^2)^{-3/2}$. At $\bar Z$, use $\sum_j(\bar Z-Z_j)=0$ and $1-(1+s)^{-1/2}\le s/2$ to obtain
\[
    \norm{\nabla F_\tau(\bar Z)}
    \le \frac1{2k\tau^2}\sum_{j=1}^k\norm{Z_j-\bar Z}^3
    \le \frac{D_C^3}{2\tau^2}.
\]
Strong gradient monotonicity between $\bar Z$ and the minimizer gives
$c_\tau\norm{\hat\mu_\tau-\bar Z}\le\norm{\nabla F_\tau(\bar Z)}$. Therefore
\begin{equation*}
    \norm{\hat\mu_\tau-\bar Z}
    \le \frac{D_C^3}{2\tau^2}
    \left(1+\frac{D_C^2}{\tau^2}\right)^{3/2},
\end{equation*}
which proves the $O(\tau^{-2})$ claim and the stated explicit bound.
\end{proof}

\subsection{Hilbert angle lemma and majority proofs}\label{app:proof-majority}

\begin{lemma}[Angle of a point outside a Hilbert ball]\label{lem:angle}
Let $D=\norm{\theta-\mu}>r$ and $v=(\theta-\mu)/D$. For every $z$ with $\norm{z-\mu}\le r$,
\[
    \left\langle v,\frac{\theta-z}{\norm{\theta-z}}\right\rangle
    \ge\sqrt{1-r^2/D^2}.
\]
\end{lemma}

\begin{proof}
Write $z-\mu=av+b$, where $b\perp v$. Then $a^2+\norm b^2\le r^2$ and $D-a>0$. The squared directional cosine is
\[
    \frac{(D-a)^2}{(D-a)^2+\norm b^2}.
\]
It is enough to prove $r^2(D-a)^2\ge(D^2-r^2)\norm b^2$. Since $\norm b^2\le r^2-a^2$, this result follows from
\begin{align*}
&r^2(D-a)^2-(D^2-r^2)(r^2-a^2)\\
&\hspace{25mm}=(Da-r^2)^2\ge0.
\end{align*}
Taking square roots proves the claim.
\end{proof}

\begin{proof}[Proof of Theorem~\ref{thm:deterministic}]
Let $\hat\theta=\hat\mu_\tau$ and suppose $D=\norm{\hat\theta-\mu}>2r$. Put $v=(\hat\theta-\mu)/D$. For a good block, $t_j=\norm{\hat\theta-Z_j}\ge D-r>r\ge\tau$. Lemma~\ref{lem:angle} gives a directional cosine larger than $\sqrt3/2$.

For canonical Huber, $\norm{\Psi^{\mathrm H}_\tau(\hat\theta-Z_j)}=\tau$. For pseudo-Huber,
\[
    \norm{\Psi^{\mathrm P}_\tau(\hat\theta-Z_j)}
    =\frac{t_j}{\sqrt{1+t_j^2/\tau^2}}
    =\tau\frac{t_j}{\sqrt{\tau^2+t_j^2}}
    \ge\frac\tau{\sqrt2}.
\]
Thus, for either loss,
\[
    \ip{v}{\Psi_\tau(\hat\theta-Z_j)}
    >\tau\sqrt{3/8}
    \qquad(j\in\mathcal G).
\]
For every bad block, bounded score gives $\ip{v}{\Psi_\tau(\hat\theta-Z_j)}\ge-\tau$. Projecting \eqref{eq:score} onto $v$ therefore yields
\[
    0=\frac1k\sum_{j=1}^k
    \ip{v}{\Psi_\tau(\hat\theta-Z_j)}
    >\tau\left\{\frac58\sqrt{\frac38}-\frac38\right\}>0,
\]
a contradiction. Hence $D\le2r$.
\end{proof}

\begin{proof}[Proof of Theorem~\ref{thm:high-prob}]
Independence and zero means eliminate the cross terms. Therefore, for each block,
\[
    \E\norm{Z_j-\mu}^2
    =\frac1m\E\norm{X-\mu}^2
    =\frac{\sigma^2}{m}.
\]
Chebyshev's inequality gives
\[
    \Pp\left\{\norm{Z_j-\mu}>2\frac\sigma{\sqrt m}\right\}\le\frac14.
\]
Let $I_j=\ind\{\norm{Z_j-\mu}>2\sigma/\sqrt m\}$. The $I_j$ are independent, and $\E I_j\le1/4$. Hoeffding's inequality gives
\[
    \Pp\left\{\frac1k\sum_{j=1}^k I_j>\frac38\right\}
    \le\exp\left[-2k\left(\frac38-\frac14\right)^2\right]
    =e^{-k/32}.
\]
On the complementary event, at least $5k/8$ block means are within radius $r=2\sigma/\sqrt m$. Theorem~\ref{thm:deterministic} gives the bound $2r=4\sigma/\sqrt m$.
\end{proof}

\subsection{General majority radii}\label{app:general-majority}

The simple $5/8$ theorem gives one statement for both losses. The Hilbert angle lemma also yields fraction-dependent bounds.

\begin{proposition}[Canonical Huber, arbitrary good fraction]\label{prop:canonical-general}
Suppose at least $(1-\eta)k$ block means lie in $B(\mu,r)$, where $0\le\eta<1/2$. For canonical Huber with any $\tau>0$,
\[
    \norm{\hat\mu_\tau-\mu}
    \le
    \max\left\{r+\tau,
    \frac{1-\eta}{\sqrt{1-2\eta}}\,r\right\}.
\]
\end{proposition}

\begin{proof}
Let $D=\norm{\hat\mu_\tau-\mu}$. If $D\le r+\tau$, the result is immediate. Otherwise, every good residual exceeds $\tau$. Each good score then has norm $\tau$. Lemma~\ref{lem:angle} and the projected score equation imply
\[
    (1-\eta)\sqrt{1-r^2/D^2}\le\eta.
\]
Solving this inequality gives $D\le(1-\eta)r/\sqrt{1-2\eta}$.
\end{proof}

\begin{proposition}[Pseudo-Huber, arbitrary good fraction]\label{prop:pseudo-general}
Under the same good-block condition, let $\omega=\eta/(1-\eta)$. The pseudo-Huber solution satisfies $\norm{\hat\mu_\tau-\mu}\le R_{\eta,\tau}$, where $R_{\eta,\tau}>r$ is the unique solution of
\begin{equation}\label{eq:pseudo-radius}
    \frac{R_{\eta,\tau}-r}{\sqrt{\tau^2+(R_{\eta,\tau}-r)^2}}
    \sqrt{1-\frac{r^2}{R_{\eta,\tau}^2}}=\omega.
\end{equation}
For $\eta=0$, interpret $R_{\eta,\tau}=r$. The left-hand side is continuous and strictly increasing from $0$ to $1$, so $R_{\eta,\tau}$ is finite for every $\eta<1/2$.
\end{proposition}

\begin{proof}
For $D>r$, a good residual is at least $D-r$. The pseudo-Huber score magnitude is increasing in the residual, and Lemma~\ref{lem:angle} controls its direction. Projecting the score equation gives
\[
    (1-\eta)\tau
    \frac{D-r}{\sqrt{\tau^2+(D-r)^2}}
    \sqrt{1-r^2/D^2}
    \le\eta\tau.
\]
The claimed bound follows by monotonicity. Take $\eta\le3/8$ and $\tau\le r$. At $R_{\eta,\tau}=2r$, the left side of \eqref{eq:pseudo-radius} is at least $\sqrt{3/8}>3/5$. This result recovers Theorem~\ref{thm:deterministic}.
\end{proof}

\begin{remark}[Constants and sharpness]\label{rem:constants}
The pair $(5/8,2r)$ is convenient for both losses, but it is not a minimax claim. At distance $2r$ with $\tau=r$, the pseudo-Huber argument holds for any good fraction larger than $\{1+\sqrt{3/8}\}^{-1}\approx0.6202$. The fraction $5/8$ is a simple rational rounding. For canonical Huber, the corresponding angle factor is $\sqrt3/2$. The same radius then permits a bad fraction below $2\sqrt3-3\approx0.464$. Proposition~\ref{prop:canonical-general} gives the sharper canonical radius. Proposition~\ref{prop:pseudo-general} gives the exact implicit pseudo-Huber radius. The exponent $1/32$ in Theorem~\ref{thm:high-prob} follows from one convenient Chebyshev and Hoeffding choice. We do not claim that it is optimal. Finally, as $\tau\downarrow0$, Proposition~\ref{prop:canonical-general} reduces exactly to
\[
    C_\eta r=\frac{1-\eta}{\sqrt{1-2\eta}}\,r.
\]
This expression is the Hilbert geometric-median constant $C_\alpha r$ in \citet[Lemma~2.1(a)]{minsker_2015_GeometricMedianRobust}, with $\alpha=\eta$.
\end{remark}

\subsection{Inference details}\label{app:inference}

Throughout this subsection, $D^r f(x)$ denotes the $r$th Fr\'echet derivative, with its usual multilinear operator norm. For pseudo-Huber,
\[
    \phi_\lambda(y)=\frac{y}{\sqrt{1+\norm y^2/\lambda^2}}.
\]
Writing $a(y)=(1+\norm y^2/\lambda^2)^{-1/2}$ gives
\begin{equation}\label{eq:pseudo-hessian}
    H_\lambda(y)=D\phi_\lambda(y)
    =a(y)I-\frac{a(y)^3}{\lambda^2}y\otimes y.
\end{equation}

\begin{lemma}[Derivative bounds]\label{lem:derivative-bounds}
For every $y,v\in\Hcal$,
\begin{equation}\label{eq:hessian-bounds}
    a(y)^3\norm v^2
    \le\ip{v}{H_\lambda(y)v}
    \le\norm v^2.
\end{equation}
Moreover,
\[
    \sup_y\norm{D^2\phi_\lambda(y)}\le\frac6\lambda,
    \qquad
    \sup_y\norm{D^3\phi_\lambda(y)}\le\frac{36}{\lambda^2}.
\]
In particular, $H_\lambda$ is globally $6/\lambda$-Lipschitz in operator norm.
\end{lemma}

\begin{proof}
The eigenvalue of $H_\lambda(y)$ along $y$ is $a(y)^3$. Every orthogonal eigenvalue is $a(y)$, which proves \eqref{eq:hessian-bounds}. Direct differentiation gives, for unit $h_1,h_2$,
\begin{align*}
D^2\phi_\lambda(y)[h_1,h_2]
={}&-\frac{a^3}{\lambda^2}
 \{\ip{y}{h_2}h_1+\ip{y}{h_1}h_2\}\\
&-\frac{a^3}{\lambda^2}\ip{h_1}{h_2}y\\
&+\frac{3a^5}{\lambda^4}\ip{y}{h_1}\ip{y}{h_2}y.
\end{align*}
With $s=\norm y/\lambda$, its norm is at most
$\lambda^{-1}\{3s(1+s^2)^{-3/2}+3s^3(1+s^2)^{-5/2}\}\le6/\lambda$.
Differentiating the first displayed group again contributes at most $\lambda^{-2}\{3a^3+9a^5s^2\}\le12/\lambda^2$. Differentiating the second contributes at most $\lambda^{-2}\{9a^5s^2+15a^7s^4\}\le24/\lambda^2$. These bounds prove the third-derivative bound and the Lipschitz claim.
\end{proof}

\begin{lemma}[Uniform population Hessian floor]\label{lem:population-floor}
Let $Y$ be centered with $\E\norm Y^2=\sigma^2$, and let
$M_{\sigma,\lambda}=2\sigma+\lambda$. If $\norm u\le M_{\sigma,\lambda}$, then, for every $t\in[0,1]$,
\begin{equation}\label{eq:population-floor}
\begin{aligned}
    \E H_\lambda(Y-tu)&\succeq c_{\sigma,\lambda}I,\\
    c_{\sigma,\lambda}
    &:=\left[1+\frac{\sigma^2+M_{\sigma,\lambda}^2}{\lambda^2}\right]^{-3/2}.
\end{aligned}
\end{equation}
\end{lemma}

\begin{proof}
By \eqref{eq:hessian-bounds}, the left side is bounded below by $\E(1+\norm{Y-tu}^2/\lambda^2)^{-3/2}I$. The scalar map $s\mapsto(1+s/\lambda^2)^{-3/2}$ is convex and decreasing. Centering gives $\E\norm{Y-tu}^2=\sigma^2+t^2\norm u^2\le\sigma^2+M_{\sigma,\lambda}^2$. Jensen's inequality proves \eqref{eq:population-floor}.
\end{proof}

\begin{proof}[Proof of Theorem~\ref{thm:fixed-m-inference}]
For fixed $m$, abbreviate $Y_j=Y_{m,j}$, $u_0=u_{m,\lambda}$, $A=A_{m,\lambda}$, and $B=B_{m,\lambda}$. The population objective $Q(u)=\E p_\lambda(\norm{Y_1-u})$ is norm-continuous, convex, weakly lower semicontinuous, and coercive. Coercivity follows from $p_\lambda(t)\ge\lambda t-\lambda^2$. A Hilbert space is reflexive, so weakly compact sublevel sets give attainment. Strict convexity gives the unique minimizer $u_0$.

Let
\[
    S_k(u)=\frac1k\sum_{j=1}^k\phi_\lambda(Y_j-u),
\]
the negative empirical gradient. The bounds $\norm{\phi_\lambda}\le\lambda$ and $\norm{H_\lambda}_{\op}\le1$ justify differentiation under the expectation. Dominated convergence gives $\E\phi_\lambda(Y_1-u_0)=0$ and $S_k(\hat u)=0$. The population comparison is $Q(u_0)\le Q(0)$. Combine it with $p_\lambda(t)\ge\lambda t-\lambda^2$ and $p_\lambda(t)\le\lambda t$. This gives $\norm{u_0}\le M_{\sigma,\lambda}$. Hence, \eqref{eq:hessian-bounds} and Lemma~\ref{lem:population-floor} yield
\[
    A\succeq cI,
    \qquad
    c:=\E[a(Y_1-u_0)^3]\ge c_{\sigma,\lambda}>0.
\]

First, $\E\norm{S_k(u_0)}^2\le\lambda^2/k$. Second,
\[
    \hat A_k(u_0)=\frac1k\sum_{j=1}^kH_\lambda(Y_j-u_0)\longrightarrow A
\]
in operator norm. In \eqref{eq:pseudo-hessian}, the scalar multiples of $I$ converge by the scalar law of large numbers. The rank-one term is Hilbert-Schmidt-valued with norm
\[
    \frac{(\norm y/\lambda)^2}{(1+\norm y^2/\lambda^2)^{3/2}},
\]
which is uniformly bounded. The Hilbert-space law of large numbers gives Hilbert-Schmidt convergence and therefore operator convergence.

Choose $\epsilon>0$ so that $6\epsilon/\lambda\le c/4$. With probability tending to one, $\norm{\hat A_k(u_0)-A}_{\op}\le c/4$. The empirical Hessian is then at least $cI/2$ throughout $B(u_0,\epsilon)$. On its boundary, write $Q_k$ for the empirical objective and use $\nabla Q_k=-S_k$ to obtain
\[
    \ip{\nabla Q_k(u_0+h)}{h}
    \ge-\norm{S_k(u_0)}\epsilon+(c/2)\epsilon^2>0.
\]
Convexity then places the global empirical minimizer inside the ball. The score equation and the same Hessian floor yield
$\norm{\hat u-u_0}\le2c^{-1}\norm{S_k(u_0)}=O_p(k^{-1/2})$.

The integral Taylor identity is
\begin{align*}
    0&=S_k(u_0)-\bar A_k(\hat u-u_0),\\
    \bar A_k
    &=\int_0^1\frac1k\sum_{j=1}^k
      H_\lambda\{Y_j-u_0-t(\hat u-u_0)\}\,dt.
\end{align*}
Lemma~\ref{lem:derivative-bounds}, operator convergence at $u_0$, and $\hat u-u_0=o_p(1)$ imply $\bar A_k\to A$ in operator norm. Thus
\[
    \sqrt k(\hat u-u_0)
    =A^{-1}\frac1{\sqrt k}\sum_{j=1}^k
      \phi_\lambda(Y_j-u_0)+o_p(1).
\]
The summands are iid, mean zero, and bounded in a separable Hilbert space. The Hilbert-space CLT therefore gives covariance $B$. Since $\sqrt n(\hat\mu-\theta_{m,\lambda})=\sqrt k(\hat u-u_0)$, the theorem follows. Central symmetry makes the population objective symmetric and strictly convex. Hence $u_0=0$.
\end{proof}

\begin{proposition}[Consistency of the fixed-block sandwich]\label{prop:sandwich-consistency}
Under Theorem~\ref{thm:fixed-m-inference},
\[
    \norm{\hat A_{m,\lambda}-A_{m,\lambda}}_{\op}\to_p0,
    \qquad
    \norm{\hat B_{m,\lambda}-B_{m,\lambda}}_1\to_p0,
\]
and therefore $\norm{\hat V_{m,\lambda}-V_{m,\lambda}}_1\to_p0$.
\end{proposition}

\begin{proof}
The assertion for $\hat A$ follows from the preceding operator law of large numbers, $\hat u-u_0=o_p(1)$, and Lemma~\ref{lem:derivative-bounds}. At $u_0$, the rank-one variables $\phi_\lambda(Y_j-u_0)\otimes\phi_\lambda(Y_j-u_0)$ have trace norm at most $\lambda^2$. Thus the law of large numbers holds in the separable trace-class Banach space. The inequality
\[
\norm{x\otimes x-y\otimes y}_1
\le(\norm x+\norm y)\norm{x-y}
\]
and the $1$-Lipschitz property of $\phi_\lambda$ control the plug-in change. Replacing $u_0$ with $\hat u$ changes the average by at most $2\lambda\norm{\hat u-u_0}$. Continuity of inversion completes the proof.
\end{proof}

\begin{lemma}[Smooth-test Hilbert Gaussian replacement]\label{lem:smooth-normal}
Let $\xi_1,\ldots,\xi_m$ be iid centered elements of a separable Hilbert space with $\E\norm{\xi_1}^3<\infty$. Put $Y_m=m^{-1/2}\sum_i\xi_i$. Let $G$ be centered Gaussian with the same covariance operator as $\xi_1$. Suppose $f:\Hcal\to\Hcal$ is three times Fr\'echet differentiable and $M_3=\sup_x\norm{D^3f(x)}<\infty$. Then
\begin{equation}\label{eq:smooth-normal}
    \norm{\E f(Y_m)-\E f(G)}
    \le\frac{M_3}{6\sqrt m}
    \{\E\norm{\xi_1}^3+\E\norm G^3\}.
\end{equation}
\end{lemma}

\begin{proof}
Let $G_1,\ldots,G_m$ be iid copies of $G$, independent of the $\xi_i$. Then $m^{-1/2}\sum_iG_i$ has law $G$. The bound on $D^3f$ implies $\norm{D^2f(x)}\le\norm{D^2f(0)}+M_3\norm x$. The functions $Df(x)$ and $f(x)$ grow at most quadratically and cubically, respectively. Every hybrid sum below has a finite third moment. Thus all conditional expectations in the Taylor expansion are well defined. For $i=1,\ldots,m$, condition on
\[
T_i=m^{-1/2}(G_1+\cdots+G_{i-1}+\xi_{i+1}+\cdots+\xi_m)
\]
and compare $f(T_i+m^{-1/2}\xi_i)$ with $f(T_i+m^{-1/2}G_i)$. Taylor's formula through order two bounds the remainder by $M_3\norm z^3/6$ for increment $z$. Conditional constant and linear terms cancel because both increments are centered. The quadratic terms also agree. For arbitrary $w\in\Hcal$, a bounded operator represents the scalar bilinear form $(x,y)\mapsto\ip{D^2f(T_i)[x,y]}{w}$. Its expected diagonal is the trace pairing with the common covariance operator. Sum the $m$ replacement errors, each scaled by $m^{-3/2}$, to prove \eqref{eq:smooth-normal}.
\end{proof}

\begin{proof}[Proof of Proposition~\ref{prop:primitive-bias}]
Write $\xi=X-\mu$, $\beta_3=\E\norm\xi^3$, and $\sigma^2=\E\norm\xi^2$. First, the population comparison is
$\E p_\lambda(\norm{Y_m-u_{m,\lambda}})\le\E p_\lambda(\norm{Y_m})$.
Combine it with $p_\lambda(t)\ge\lambda t-\lambda^2$ and $p_\lambda(t)\le\lambda t$. This gives
\begin{equation}\label{eq:um-uniform}
    \norm{u_{m,\lambda}}\le M_{\sigma,\lambda}:=2\sigma+\lambda.
\end{equation}
Apply Lemma~\ref{lem:population-floor} with $Y=Y_m$ and $u=u_{m,\lambda}$. It gives a common lower bound along the segment from $0$ to $u_{m,\lambda}$. The bounded score justifies the population first-order condition $\E\phi_\lambda(Y_m-u_{m,\lambda})=0$. The fundamental theorem of calculus gives the identity
\begin{align}
\E\phi_\lambda(Y_m)
&=\E\{\phi_\lambda(Y_m)-\phi_\lambda(Y_m-u_{m,\lambda})\}\notag\\
&=\left\{\int_0^1\E H_\lambda(Y_m-tu_{m,\lambda})\,dt\right\}u_{m,\lambda}.
\label{eq:population-score-identity}
\end{align}
Combining \eqref{eq:population-score-identity} with \eqref{eq:population-floor} yields
\begin{equation*}
    c_{\sigma,\lambda}\norm{u_{m,\lambda}}
    \le\norm{\E\phi_\lambda(Y_m)}.
\end{equation*}

Apply Lemma~\ref{lem:smooth-normal} with $f=\phi_\lambda$. Lemma~\ref{lem:derivative-bounds} gives $M_3\le36/\lambda^2$. Gaussian symmetry gives $\E\phi_\lambda(G)=0$. Moreover, $\E\norm G^4=(\tr\Gamma)^2+2\norm{\Gamma}_{\mathcal S_2}^2\le3\sigma^4$. Hence $\E\norm G^3\le\sqrt3\,\sigma^3\le\sqrt3\,\beta_3$. Thus
\[
    \norm{u_{m,\lambda}}
    \le \frac{6(1+\sqrt3)\beta_3}
    {c_{\sigma,\lambda}\lambda^2\sqrt m},
\]
which proves the primitive bias rate.

The Hilbert CLT gives $Y_m\Rightarrow G$, and the bias bound gives $u_{m,\lambda}\to0$. The map $y\mapsto\phi_\lambda(y)\otimes\phi_\lambda(y)$ is continuous into the trace class. Its trace norm is bounded by $\lambda^2$. The map $y\mapsto H_\lambda(y)$ is operator-norm continuous and bounded by one. On a Skorokhod coupling, $Y_m-u_{m,\lambda}\to G$ almost surely. Dominated convergence therefore yields
$B_{m,\lambda}\to B_\lambda^G$ in trace norm and
$A_{m,\lambda}\to A_\lambda^G$ in operator norm. Finally, \eqref{eq:hessian-bounds} gives
$A_\lambda^G\succeq\E(1+\norm G^2/\lambda^2)^{-3/2}I$.
Therefore, the Gaussian-limit Hessian is boundedly invertible.
\end{proof}

\begin{proof}[Proof of Corollary~\ref{cor:triangular}]
We give the array-uniform localization and linearization. Put $u_m=u_{m,\lambda}$ and
\[
S_{k,m}(u)=\frac1k\sum_{j=1}^k\phi_\lambda(Y_{m,j}-u).
\]
The uniform bound \eqref{eq:um-uniform} and Lemma~\ref{lem:population-floor} imply
$A_{m,\lambda}\succeq c_{\sigma,\lambda}I$ for every $m$. Also,
\begin{equation}\label{eq:uniform-score}
    \E\norm{S_{k,m}(u_m)}^2\le\frac{\lambda^2}{k}.
\end{equation}
At $u_m$, define
\[
    \hat A_{k,m}(u_m)=\frac1k\sum_{j=1}^k
    H_\lambda(Y_{m,j}-u_m)
\]
and decompose it using \eqref{eq:pseudo-hessian}. The scalar average has variance at most $1/k$. The rank-one part is Hilbert-Schmidt-valued. Its Hilbert-Schmidt norm is bounded by
$s^2(1+s^2)^{-3/2}$ with $s=\norm{Y_{m,j}-u_m}/\lambda$. Hence, for a constant depending only on $\lambda$,
\begin{equation}\label{eq:uniform-hessian}
    \E\norm{\hat A_{k,m}(u_m)-A_{m,\lambda}}_{\op}^2
    \le\frac{C_\lambda}{k},
\end{equation}
uniformly in $m$.

Choose the fixed radius $\epsilon=c_{\sigma,\lambda}\lambda/24$. Apply Lemma~\ref{lem:derivative-bounds}, \eqref{eq:uniform-score}, and \eqref{eq:uniform-hessian}. With probability tending to one, the empirical Hessian is uniformly bounded below by $c_{\sigma,\lambda}I/2$ throughout $B(u_m,\epsilon)$. The outward gradient is also positive on the boundary. Convexity localizes $\hat u$ in that ball and gives
\[
    \norm{\hat u-u_m}
    \le2c_{\sigma,\lambda}^{-1}\norm{S_{k,m}(u_m)}
    =O_p(k^{-1/2}).
\]
The integral Taylor identity and the $6/\lambda$ Hessian Lipschitz bound now yield, in Hilbert norm,
\begin{equation}\label{eq:uniform-linearization}
    \sqrt k(\hat u-u_m)
    =A_{m,\lambda}^{-1}\frac1{\sqrt k}\sum_{j=1}^k
      \phi_\lambda(Y_{m,j}-u_m)+o_p(1).
\end{equation}
All remainder bounds above are uniform in $m$.

For a fixed finite-rank $L:\Hcal\to\R^q$, the summands remain uniformly bounded after applying $L A_{m,\lambda}^{-1}$. Proposition~\ref{prop:primitive-bias} identifies their covariance limit as $L V_\lambda^G L^*$. Apply the finite-dimensional triangular-array Lindeberg theorem to \eqref{eq:uniform-linearization}. This gives the projected CLT around $\theta_{m,\lambda}$. Finally,
\[
    \sqrt n\,L(\theta_{m,\lambda}-\mu)
    =\sqrt k\,L u_m
    =O\!\left(\sqrt{k/m}\right)=o(1),
\]
which centers the limit at $\mu$.

For projected sandwich consistency, combine \eqref{eq:uniform-hessian}, the $6/\lambda$ plug-in bound, and \eqref{eq:uniform-linearization}. They give $\norm{\hat A_{m,\lambda}-A_{m,\lambda}}_{\op}=O_p(k^{-1/2})$. The uniform floor therefore implies
$\norm{\hat A_{m,\lambda}^{-1}-A_{m,\lambda}^{-1}}_{\op}=O_p(k^{-1/2})$.
Since $\norm{\hat B_{m,\lambda}}_1\le\lambda^2$, replace both random inverse factors with $A_{m,\lambda}^{-1}$. This replacement changes the projected sandwich by $O_p(k^{-1/2})$.

Next define
\[
    \widetilde B_{k,m}
    =\frac1k\sum_{j=1}^k
      \phi_\lambda(Y_{m,j}-u_m)\otimes
      \phi_\lambda(Y_{m,j}-u_m).
\]
Each entry of $L A_{m,\lambda}^{-1}\widetilde B_{k,m}A_{m,\lambda}^{-1}L^*$ averages row-iid, uniformly bounded scalar variables. Its centered variance is therefore $O(1/k)$. Finally, the rank-one trace inequality and the $1$-Lipschitz score give
$\norm{\hat B_{m,\lambda}-\widetilde B_{k,m}}_1
\le2\lambda\norm{\hat u-u_m}=O_p(k^{-1/2})$.
Proposition~\ref{prop:primitive-bias} and continuity of inversion now yield the asserted convergence to $L V_\lambda^G L^*$.
\end{proof}

\subsection{Large-threshold efficiency}\label{app:large-threshold}

The operators $A_\lambda,B_\lambda,V_\lambda$ in this subsection apply to a generic centered variable $Y$. They differ from the Gaussian-block operators $A_\lambda^G,B_\lambda^G,V_\lambda^G$ in Section~6.

\begin{proposition}[Pseudo-Huber covariance approaches the mean covariance]\label{prop:large-threshold}
Let $Y$ be a centered Hilbert-valued random element with $\E\norm Y^3<\infty$. Let $u_\lambda$ minimize $\E p_\lambda(\norm{Y-u})$. Define $A_\lambda$, $B_\lambda$, and $V_\lambda=A_\lambda^{-1}B_\lambda A_\lambda^{-1}$ at $u_\lambda$. Then
\begin{align*}
    \norm{u_\lambda}&=O(\lambda^{-2}),\\
    A_\lambda&\to I
    &&\text{in operator norm},\\
    B_\lambda&\to\E(Y\otimes Y)
    &&\text{in trace norm}.
\end{align*}
Consequently, $V_\lambda\to\E(Y\otimes Y)$ in trace norm.
\end{proposition}

\begin{proof}
Write $w_\lambda(v)=(1+\norm v^2/\lambda^2)^{-1/2}$. The score equation is $\E[w_\lambda(Y-u_\lambda)(Y-u_\lambda)]=0$. Put $s_\lambda=\norm{u_\lambda}$ and $\sigma_Y^2=\E\norm Y^2$. Taking the inner product with $u_\lambda$ gives $s_\lambda\E w_\lambda(Y-u_\lambda)\le\sigma_Y$. Since $(1+x)^{-1/2}$ is convex, Jensen's inequality yields
\[
    \E w_\lambda(Y-u_\lambda)
    \ge\left\{1+(\sigma_Y^2+s_\lambda^2)/\lambda^2\right\}^{-1/2}.
\]
Thus $s_\lambda^2\le3\sigma_Y^2$ whenever $\lambda^2\ge2\sigma_Y^2$. This gives the required uniform bound. Also,
\[
    1-w_\lambda(v)\le\frac{\norm v^2}{2\lambda^2}.
\]
Since $\E(Y-u_\lambda)=-u_\lambda$,
\[
    u_\lambda
    =\E\{[w_\lambda(Y-u_\lambda)-1](Y-u_\lambda)\}.
\]
The uniform bound and the finite third moment therefore give $\norm{u_\lambda}=O(\lambda^{-2})$.

From \eqref{eq:pseudo-hessian},
\[
    \norm{I-H_\lambda(v)}_{\op}
    \le \frac{3\norm v^2}{2\lambda^2},
\]
which yields $A_\lambda\to I$. Moreover, $\phi_\lambda(Y-u_\lambda)\to Y$ in $L^2$. The convergence is pointwise, and the centers converge to zero. For all sufficiently large $\lambda$,
\[
    \norm{\phi_\lambda(Y-u_\lambda)-Y}
    \le 2\norm Y+\sup_{\lambda\ge\lambda_0}\norm{u_\lambda},
\]
whose square is integrable by the uniform center bound and $\E\norm Y^2<\infty$. The standard rank-one inequality
\[
    \norm{x\otimes x-y\otimes y}_1
    \le(\norm x+\norm y)\norm{x-y}
\]
then gives trace-norm convergence of $B_\lambda$. Continuity of inversion completes the proof.
\end{proof}

\subsection{Majorization-minimization details}\label{app:algorithm}

\begin{proposition}[Monotone HOMER updates]\label{prop:mm}
For either canonical or pseudo-Huber loss, the update \eqref{eq:mm-update} satisfies
$F_\tau(\theta^{(t+1)})\le F_\tau(\theta^{(t)})$. For pseudo-Huber, the iterates converge to the unique HOMER minimizer.
\end{proposition}

\begin{proof}
For $g_\tau(s)=\rho_\tau(\sqrt s)$, direct differentiation shows that $g_\tau$ is concave on $[0,\infty)$. Hence
\[
    g_\tau(s)\le g_\tau(s_0)+g_\tau'(s_0)(s-s_0).
\]
Applying this inequality to each squared residual at $\theta^{(t)}$ gives a quadratic majorizer. Its minimizer is exactly \eqref{eq:mm-update}. The majorization inequality proves monotone descent. All iterates lie in the compact convex hull $C$ of the block means. For pseudo-Huber, the weights are continuous and uniformly positive on the bounded residual range generated by $C$. Strong convexity of each quadratic surrogate therefore yields, for some $c_C>0$,
\[
    F_\tau(\theta^{(t)})-F_\tau(\theta^{(t+1)})
    \ge c_C\norm{\theta^{(t+1)}-\theta^{(t)}}^2.
\]
Thus successive differences vanish. Suppose a subsequence converges to $\theta^\star$. By continuity, its one-step successors converge to both $\theta^\star$ and its update. Hence $\theta^\star$ is a fixed point and satisfies the score equation. Strict convexity gives only one such point. Therefore, every cluster point is the unique HOMER minimizer, and the full sequence converges to it.
\end{proof}

\section{DATA-DRIVEN THRESHOLDS}\label{app:tuning}

The oracle robust threshold in Theorem~\ref{thm:high-prob} is of order $\sigma/\sqrt m$. In practice $\sigma$ is unknown. We recommend a robust block-scale estimate:
\begin{enumerate}[leftmargin=*]
    \item Compute a preliminary center $\tilde\mu$, either MOM or a very small-threshold HOMER estimate.
    \item Compute distances $d_j=\norm{Z_j-\tilde\mu}$.
    \item If the median is positive, set $\hat\tau=c\cdot \mathrm{median}_j d_j$ for $c\in\{1,2,4,8\}$.
    \item If the median vanishes but some distance is positive, use the median positive distance. If all distances vanish, return their common center.
\end{enumerate}
The released routine substitutes the positive value one in the all-zero branch. All block summaries then coincide, so its final center is identical. The recorded fallback threshold is not a data-scale estimate.
This tuning rule is not covered by Theorem~\ref{thm:high-prob}, Theorem~\ref{thm:fixed-m-inference}, or Corollary~\ref{cor:triangular}. The first theorem uses an oracle block scale, while the inference results fix $\lambda$. The first and second simulated examples vary fixed thresholds $\tau=\lambda/\sqrt m$ under normalized scales. These curves isolate threshold behavior but do not compare an oracle selector with an adaptive rule.

A Lepski-type alternative computes $\hat\mu_{\tau_\ell}$ over a grid $\tau_1<\cdots<\tau_L$. It selects the smallest $\tau_\ell$ whose estimate is stable relative to larger thresholds. This approach may formalize the robustness and efficiency tradeoff without prior scale knowledge.

\raggedbottom
\section{ADDITIONAL EXPERIMENTAL DETAILS}\label{app:experiments}

The first three studies expand the examples in the main paper. The remaining studies cover covariance operators, kernel embeddings, distributed gradients, financial returns, demand curves, and classification.

\subsection{Finite-dimensional Hilbert-mean simulation}

The full robustness profile uses $d\in\{100,500,2000\}$ and $\alpha\in\{0.75,1,1.5\}$. The paper profile fixes $d=500$ and $\alpha=1$ for primary sweeps. Inference uses $d=100$, while the paper sensitivity analysis varies $d$ with $\alpha=1$. Set $\kappa_\ell=C\ell^{-2\alpha}$ with $\sum_\ell\kappa_\ell=1$. Mean coefficients satisfy $\mu_\ell\propto\ell^{-\alpha-1/2}$ and $\norm\mu=1$. Generate
\[
    X=\mu+S K^{1/2}\xi,
\]
where $K=\operatorname{diag}(\kappa_1,\ldots,\kappa_d)$ and $\xi$ has standardized coordinates. Regimes include Gaussian, variance-normalized Student mixtures, Pareto-3, and skewed two-point coordinates. Pareto-3 has finite variance and infinite fourth moment. Complete-block contamination adds the same shift to every observation in each selected block.

The configured but unreported paper profile assigns 500 replications to primary sweeps. Its sensitivity sweep uses 100 replications. The executed reduced grid instead uses 500 replications throughout. In the paper profile, population block-Huber targets use 20,000 Monte Carlo draws. Coverage intervals exclude target simulation error. By symmetry, Gaussian, Student, and Pareto-3 targets equal the mean. Any simulated nonzero shift in those regimes is Monte Carlo error. The default $k=16$ remains fixed. Consequently, coverage does not instantiate $k,m\to\infty$. Report realized contamination $\lfloor\epsilon k\rfloor/k$, not only nominal $\epsilon$.

Figure~\ref{fig:hilbert-mean-supplement} reports additional reduced-grid diagnostics beyond Figure~\ref{fig:hilbert-mean-overview}. No method uniformly dominates in the clean-tail panel. As $c$ grows, both HOMER variants approach empirical-mean efficiency. MOM remains the median endpoint throughout the threshold path. Panel (d) reports estimated target displacement. Under symmetric laws, its population value is zero. The plotted Gaussian curve therefore measures Monte Carlo target error rather than bias.

\begin{figure*}[t]
\centering
\resizebox{\textwidth}{!}{\input{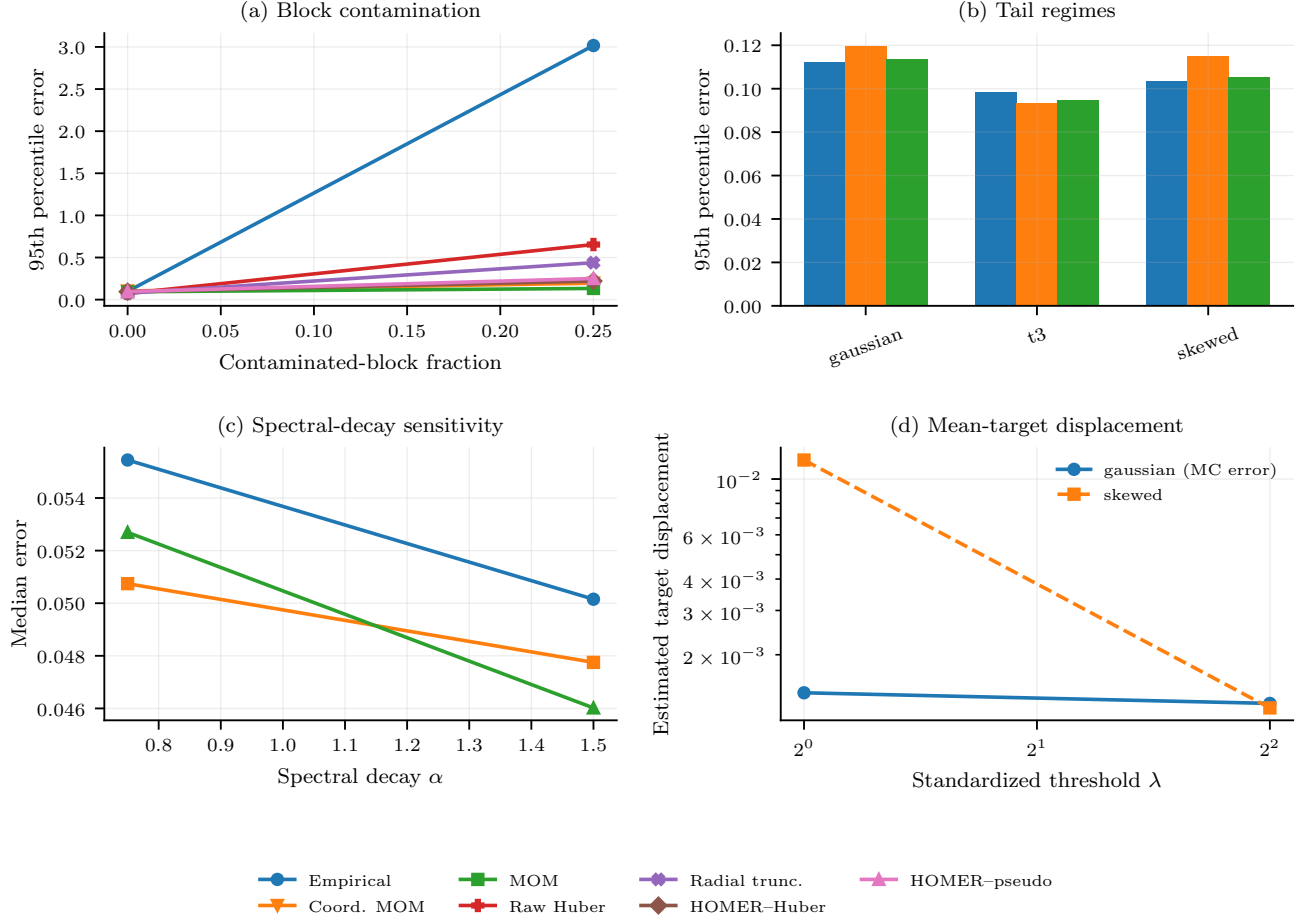}}
\caption{\textbf{Supplementary Hilbert-mean diagnostics.}
(a) Student-$t_3$ complete-block contamination. (b) Clean tail regimes. (c) Spectral-decay sensitivity on the reduced grid. (d) Monte Carlo estimate of finite-block target displacement. Under symmetric laws, the curves reflect simulation error around the theoretical value zero. Error summaries use 500 replications.}
\label{fig:hilbert-mean-supplement}
\end{figure*}

\subsection{Functional-mean simulation}

Generate random functions on $[0,1]$ in a real orthonormal Fourier basis,
\[
    X(t)=\mu(t)+S\sum_{\ell=1}^{d}\sqrt{\kappa_\ell}\xi_\ell e_\ell(t).
\]
Computations use coefficient space, where Euclidean error equals truncated $L^2$ error. Inference targets an interval average, a Gaussian-smoothed evaluation, and an evening-minus-morning contrast. Literal point evaluation is not continuous on abstract $L^2$. The paper profile sets $k=16$ and uses 15,000 Monte Carlo draws for each block-Huber target. It does not propagate target simulation error. Coverage is therefore a finite-$k$ diagnostic, not a triangular-array validation.

The executed reduced grid uses $n=256$, $d=21$, $k=8$, 500 replications, and 15,000 target draws. Figure~\ref{fig:functional-mean-curves} includes a representative Student-$t_3$ fit with two of eight blocks contaminated. Replicated results are more informative than the representative fit. The empirical 95th-percentile $L^2$ error was $2.524$, versus $0.132$ for MOM, $0.191$ for canonical HOMER, and $0.211$ for pseudo-HOMER. The coverage panel documents the observed undercoverage.

\begin{figure*}[t]
\centering
\resizebox{\textwidth}{!}{\input{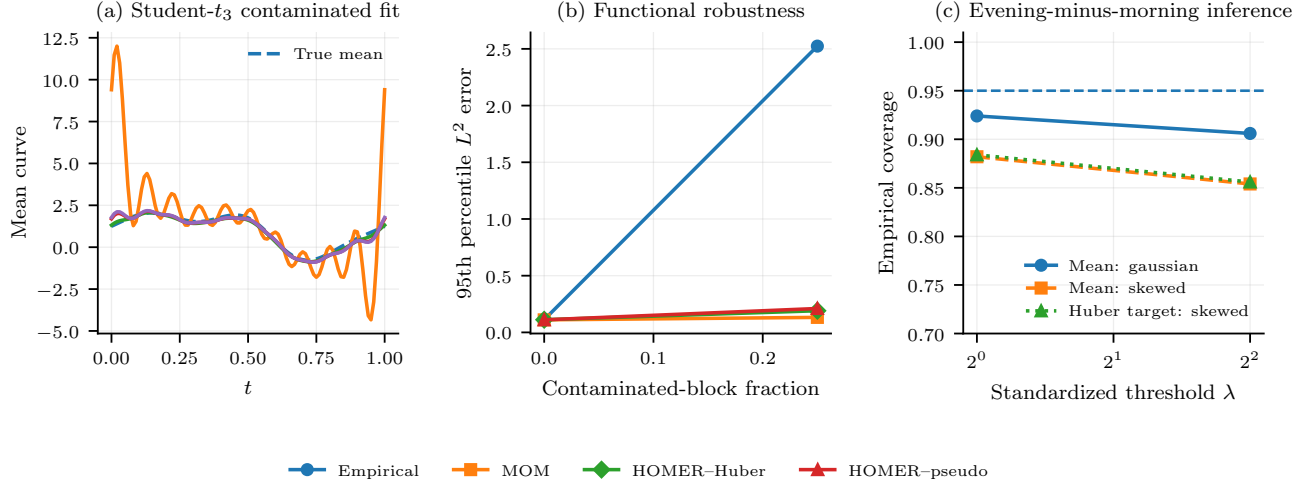}}
\caption{\textbf{Functional-mean simulation.}
(a) Representative Student-$t_3$ curve with two of eight complete blocks contaminated. (b) Replicated functional robustness over 500 Monte Carlo replications. (c) Fixed-$k$ coverage for the evening-minus-morning contrast. Literal point evaluation is not used.}
\label{fig:functional-mean-curves}
\end{figure*}

\subsection{Wearable-sensor trajectories}

The analysis uses the UCI HAR Using Smartphones data \citep{jorgereyes-ortiz_2013_HumanActivityRecognition}. Each 2.56-second window contains inertial signals sampled at 50Hz. The study retains six body-acceleration and gyroscope channels. Subjects, rather than overlapping windows, are the independent units. Windows are aligned and averaged within subject and activity before estimating the activity-specific mean trajectory.

Full activity fits use 30 subjects in five equal blocks of size six. Leave-one-subject-out fits repartition 29 subjects into five near-equal blocks. Their stability results therefore include partition variation. Corrupted subject rows and failed channels are redrawn by activity. Thus, the stress does not represent one persistent participant-level failure across activities. The energy summary first squares each subject's averaged trajectory, then averages across subjects. It is not mean window-level energy. Energy intervals use five block summaries and a data-driven threshold. We treat them as exploratory Wald diagnostics without theorem-based coverage claims. Classification removes the held-out subject from every template. Methods sample held-out windows separately. Their accuracy comparisons are therefore descriptive, not exactly paired.

At $c=2$ with 10\% subject-spike stress, MOM moved $0.283$ from the unstressed WALKING reference. Ordinary averaging moved $0.683$. Both HOMER variants moved about $0.53$. At 20\%, the four distances were already similar. At 30\%, the robust outer centers were no better. Sensor-failure classification likewise shows only small differences among estimators. Once corrupted subjects enter most blocks, robust outer aggregation cannot repair ordinary within-block means. Figure~\ref{fig:wearable-trajectories} summarizes these comparisons.

\begin{figure*}[t]
\centering
\resizebox{\textwidth}{!}{\input{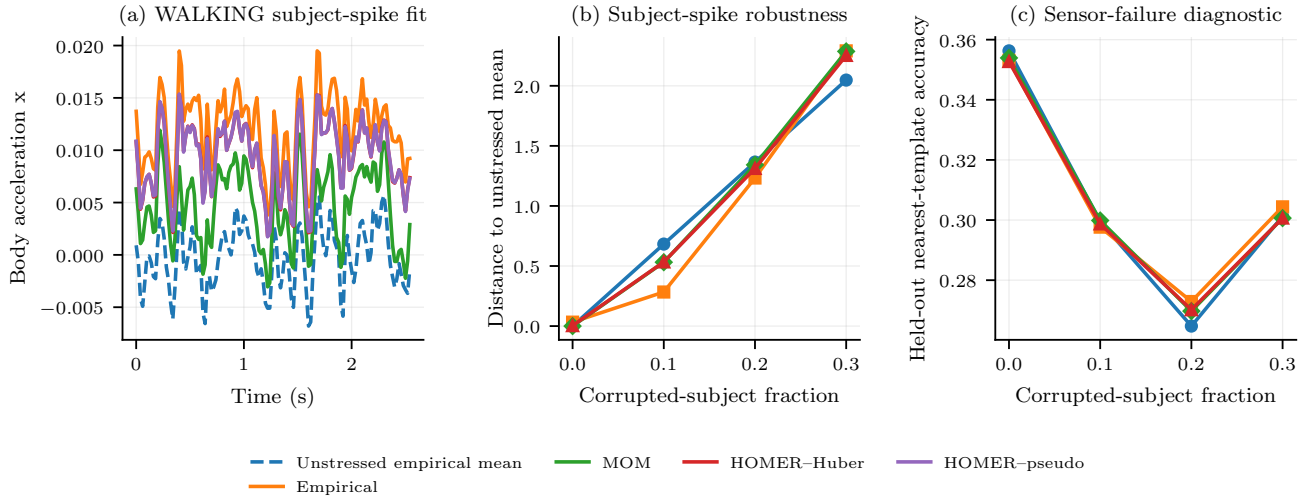}}
\caption{\textbf{Subject-level wearable trajectories.}
(a) Representative WALKING fit under subject-spike stress. (b) Distance to the unstressed empirical trajectory as corrupted subjects increase at $c=2$. (c) Held-out nearest-template accuracy under sensor failure. Panels (a) and (b) use a different stress mechanism from panel (c). All comparisons are descriptive.}
\label{fig:wearable-trajectories}
\end{figure*}

\subsection{Covariance operators and PCA}

Covariance estimation becomes a Hilbert-mean problem after lifting observations to rank-one operators. This study reports covariance error, eigenspace stability, and the rank-one centering term from an independent split.

Generate centered high-dimensional or functional observations. Form $Y_i=X_i\otimes X_i$ in Hilbert-Schmidt space. Apply empirical averaging, geometric MOM, and HOMER to block covariance summaries. For an unknown mean, estimate it on an independent split. Report $\norm{\tilde\mu-\mu}^2$, the Hilbert-Schmidt norm of its rank-one centering term. Evaluate
\[
    \norm{\hat\Sigma-\Sigma}_{\mathcal S_2},
    \qquad
    \norm{\hat\Sigma-\Sigma}_{\op},
    \qquad
    \norm{\hat P_r-P_r}_{\mathcal S_2}.
\]
Vary eigengap, spectral decay, tail index, and complete-block positive-semidefinite contamination. Pareto settings probe the fourth-moment boundary required for a covariance-operator second moment.

At the small-gap cell, the population eigengap was $0.0458$. Across 500 replications, corrupting two of eight Student-$t_5$ covariance summaries raised empirical 95th-percentile Frobenius error from $0.223$ to $3.007$. The stressed errors were $0.217$ for MOM, $0.345$ for canonical HOMER, and $0.378$ for pseudo-HOMER. MOM had the lowest stressed error under this severe positive-semidefinite spike. Increasing the eigengap from $0.0458$ to $0.0627$ reduced MOM's mean projector error from $0.915$ to $0.709$. The displayed $c=2$ HOMER fits remained near $1.42$. Stable covariance norm error need not yield a stable leading subspace when the eigengap is small.

\begin{figure*}[t]
\centering
\resizebox{\textwidth}{!}{\input{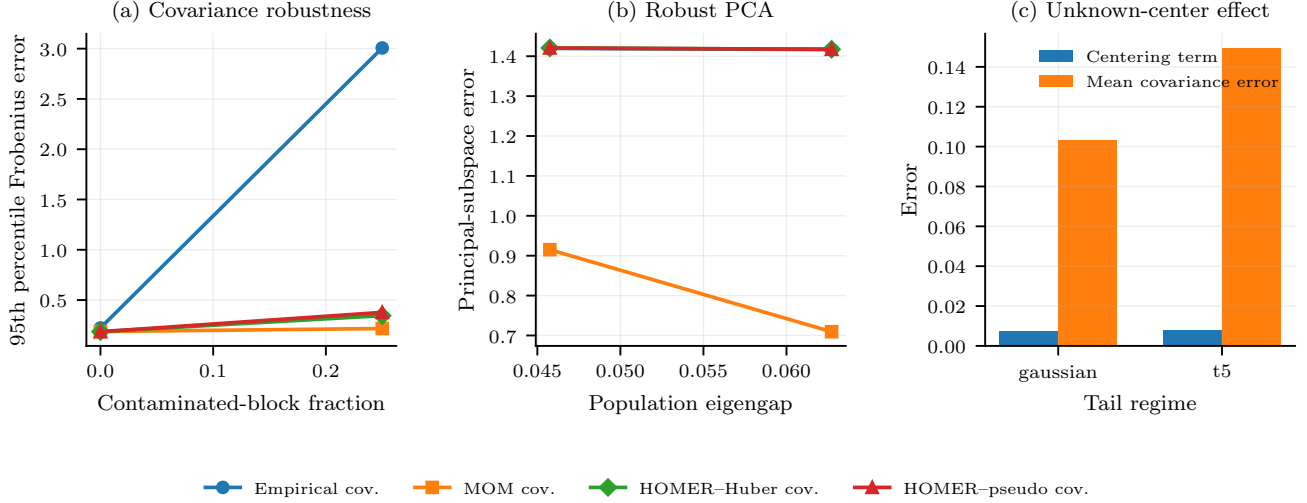}}
\caption{\textbf{Covariance operators and PCA.}
(a) Student-$t_5$ Frobenius robustness with $n=256$, $d=20$, $k=8$, known center, and population eigengap $0.0458$. (b) Rank-three projector error under two contaminated summaries. (c) Clean unknown-center decomposition by tail regime. In this diagnostic panel, ``mean covariance error'' averages only the four displayed methods. All Monte Carlo summaries use 500 replications.}
\label{fig:covariance-pca}
\end{figure*}

\subsection{RKHS mean embeddings and MMD}

Exact block Gram matrices permit outer aggregation without explicit feature coordinates. This study compares kernel-mean error with the resulting MMD-squared error.

Let $W_i\in\R^p$ and $X_i=K(W_i,\cdot)$ for Gaussian or Laplace kernels. Use linear and polynomial kernels as heavy-feature stress tests. Compute the exact $k\times k$ Gram matrix among block embeddings,
\[
    \ip{Z_j}{Z_\ell}_{\mathcal K}
    =\frac1{m^2}\sum_{i\in B_j}\sum_{i'\in B_\ell}K(W_i,W_{i'}),
\]
and perform outer aggregation using only this matrix. Compare empirical KME, a Weiszfeld geometric-MOM approximation, and HOMER. The code does not implement the min-max MONK estimator. Its epsilon-stabilized Weiszfeld step can stop at a coincident nonoptimal embedding. Exact geometric-MOM comparisons require a subgradient correction in that case. Gaussian-data RBF targets and linear-kernel targets are analytical. In the paper profile, other targets use 750 independent reference draws. Their reported MMD-squared errors include reference-sample error and V-statistic bias. Optional permutation values are descriptive, and no finite-sample level guarantee is claimed.

With two of eight RBF block embeddings replaced by outlier clusters, the empirical 95th-percentile RKHS error rose from $0.090$ to $0.311$. The stressed errors were $0.122$ for the Weiszfeld geometric-MOM approximation, $0.177$ for canonical HOMER, and $0.182$ for pseudo-HOMER. Mean absolute MMD-squared error under the same stress was $0.094$ for empirical averaging, compared with $0.016$, $0.020$, and $0.021$. The linear kernel magnified the same separation. Panel (c) times only outer aggregation after the Gram matrices have been formed. It is not a Gram-construction benchmark.

\begin{figure*}[t]
\centering
\resizebox{\textwidth}{!}{\input{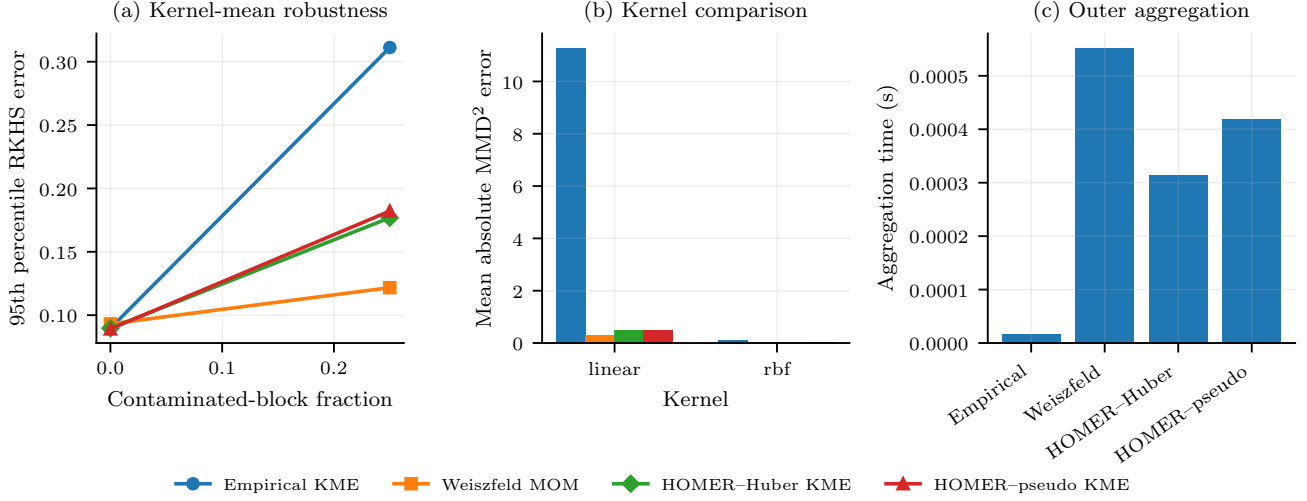}}
\caption{\textbf{Robust kernel means and MMD.}
(a) RBF embedding error under complete-block outlier clusters. (b) Mean absolute MMD-squared error across the RBF and linear kernels with two of eight blocks contaminated. (c) Clean outer-aggregation time at $n=128$. Sample and block Gram construction is excluded. The reduced grid uses 500 replications.}
\label{fig:rkhs-embeddings}
\end{figure*}

\subsection{Distributed gradient aggregation}

Worker mini-batch gradients are Hilbert-valued block summaries, directly matching the aggregation setting. A short repeated-update diagnostic records downstream behavior without asserting convergence.

At a fixed parameter $\theta_0$, draw heavy-tailed stochastic gradients with mean $g(\theta_0)$. Treat worker or mini-batch averages as block summaries. Compare empirical, geometric-median, and HOMER gradients under clean, heavy-tailed, and worker-corrupted regimes. The primary metric is $\norm{\hat g-g(\theta_0)}$. Secondary metrics are one-step objective decrease and an illustrative repeated trajectory for $f(\theta)=\norm{\theta}^2/2$.

With two of eight corrupted workers, empirical Student-$t_3$ 95th-percentile gradient error was $3.025$. The corresponding errors were $0.131$ for the geometric median, $0.188$ for canonical HOMER, and $0.208$ for pseudo-HOMER. Clean errors were between $0.102$ and $0.107$, and clean one-step objective decreases were nearly identical. Across five contaminated iterations, empirical averaging increased the mean quadratic objective from $0.5$ to $4.166$. The final mean objective was $0.00124$ for the median, $0.00503$ for canonical HOMER, and $0.00671$ for pseudo-HOMER. Each step draws fresh gradient noise and corrupted workers. This diagnostic does not establish a Byzantine optimization theorem.

\begin{figure*}[t]
\centering
\resizebox{\textwidth}{!}{\input{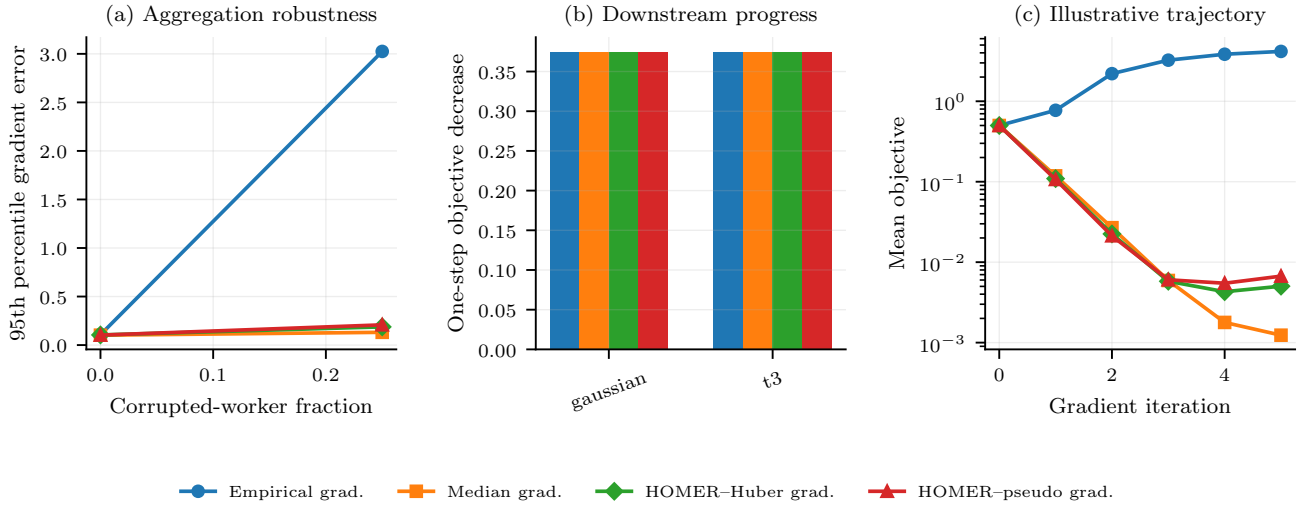}}
\caption{\textbf{Distributed gradient summaries.}
(a) Student-$t_3$ pointwise aggregation error. (b) Clean one-step progress. (c) Illustrative quadratic-objective paths with two of eight workers corrupted at each iteration. The reduced grid uses 500 replications. Panel (c) averages independently redrawn paths.}
\label{fig:distributed-gradients}
\end{figure*}

\subsection{Industry portfolio returns}

Industry returns provide an unknown-target, serially dependent setting where crisis periods can dominate means, covariances, and leading subspaces. This study evaluates stability under chronological stress and does not invoke the iid theorem.

The analysis uses daily returns for the 49 Industry Portfolios from the Kenneth French Data Library \citep{fama_1993_CommonRiskFactors}. Each day is a vector in $\R^{49}$. The study estimates mean returns, sample-split covariances, and leading components. Estimation uses 1,260-day training windows, 252-day holdouts, and contiguous blocks. Evaluation uses holdout error, long-window references, circular moving-block-bootstrap variability, and adjacent-window stability. Stress tests add crash shifts or replace raw block observations with an empirical crisis segment. The saved centering field is squared disagreement between robust and empirical first-half centers. It is not population rank-one centering error.

The 31-window averages in Figure~\ref{fig:industry-returns} use $c=2$. Selecting four of 16 training blocks raised empirical mean holdout error to $0.2185$ under the crash shift. MOM, canonical HOMER, and pseudo-HOMER remained at $0.0080$, $0.0104$, and $0.0115$. Covariance error was $0.1786$ for empirical averaging, $0.0075$ for MOM, and $0.0130$ for pseudo-HOMER. Improvements in the leading rank-three subspace were more modest. The stress fraction counts the original 16 training blocks. After covariance sample splitting, it need not equal the fraction of affected covariance summaries.

\subsection{Bike-demand curves}

Daily demand curves provide an interpretable functional target, but adjacent days are dependent and the next-window target changes. This study reports robustness and interval stability under chronological splits without claiming time-series coverage.

The analysis uses complete 24-hour UCI Bike Sharing curves after applying $\log(1+\mathrm{count})$ by default \citep{fanaee-t_2014_EventLabelingCombining}. Groups include overall, working-day, nonworking-day, holiday, nonholiday, and season. Because it uses $\mathrm{workingday}=0$, the code's ``weekend'' group is the nonworking-day group. The morning functional averages hours 7 through 9 rather than taking a maximum. Weather stress prioritizes blocks containing severe-weather days. If too few exist, selection falls back to all blocks. Every day in each selected block is rescaled. With $k=8$, nominal 10\% stress perturbs no blocks and is merged with the clean plot point. Next-window ``coverage'' compares a training interval with a noisy holdout mean and omits holdout uncertainty.

Across 15 overlapping windows, event-spike holdout error increased from $1.5803$ to $1.7653$ for the empirical mean when two of eight blocks were perturbed. The corresponding errors were $1.4939$ for MOM, $1.6044$ for canonical HOMER, and $1.5942$ for pseudo-HOMER. In the unstressed $c=2$ diagnostic, pseudo-HOMER covered the 7 to 9 a.m. average in 5 of 15 windows. It covered the 5 to 7 p.m. minus 7 to 9 a.m. contrast in 7 of 15 windows. Both rates were far below 0.95. These intervals are dependence-sensitive stability summaries, not confidence intervals for future demand.

\subsection{Class-conditional kernel embeddings}

Class-conditional kernel means connect Hilbert estimation to a nearest-embedding classifier and pairwise MMD-squared comparisons. This study varies only the outer aggregation on a subject-disjoint benchmark.

The analysis uses UCI HAR features \citep{jorgereyes-ortiz_2013_HumanActivityRecognition} with the official subject-disjoint split and training-only standardization. It estimates class-conditional RBF kernel means under label noise and outlier-cluster stress. The code equally averages block embeddings for its mean baseline. With unequal class blocks, this differs from the empirical KME. We therefore call it the \emph{equal-block mean}. The Weiszfeld geometric-MOM approximation and HOMER share each scenario's training data, partition, and Gram matrices. Clean-to-stressed displacement also includes separately generated partition variation.

At $k=16$, $c=2$, and the median-distance bandwidth, all four aggregation rules gave similar nearest-embedding accuracy. At 20\% label noise, accuracy ranged from $81.56\%$ to $81.78\%$. At 20\% outlier contamination, mean clean-to-stressed RKHS shifts ranged from $0.04405$ to $0.04428$. Canonical HOMER coincided numerically with the equal-block mean along the displayed paths. This numerical coincidence is consistent with its quadratic-basin endpoint. The WALKING-versus-WALKING\_UPSTAIRS MMD-squared shift was more sensitive to bandwidth. For the equal-block mean, it was $0.00549$, $0.000163$, and $0.000271$ at multipliers $0.5$, $1$, and $2$. Stress affects individual windows rather than a minority of complete blocks. Because correlated windows remain in the HAR split, these results provide only a descriptive sensitivity analysis.

\subsection{Metrics, plots, and reporting}

The displayed finite-dimensional, functional-mean, covariance, kernel-embedding, and distributed-gradient simulations use compact grids with 500 replications per cell. The text and manifests identify them as reduced-grid results. The larger \texttt{paper} and \texttt{full} grids remain reproducible configurations, not reported runs. The finite-dimensional simulation records mean, median, 90th-percentile, and 95th-percentile norm error. Other studies retain the task-specific summaries used by their figures and downstream analyses. A nominal whole-block contamination level $\epsilon$ is plotted at the realized fraction $\lfloor\epsilon k\rfloor/k$. Under symmetric laws, target displacement in the finite-dimensional simulation is Monte Carlo approximation error around the theoretical value zero. The bike-demand study uses ``nonworking day'' and ``7 to 9 a.m. average.'' The class-conditional kernel study names both baselines explicitly and reports shifts in squared MMD. Pseudo-Huber is the inferential estimator. Mean coverage is separated from finite-block-target coverage under skewness. Unless a stated theorem applies, real-data intervals, holdout checks, and stress paths are descriptive.

\begin{figure*}[p]
\centering
\resizebox{\textwidth}{!}{\input{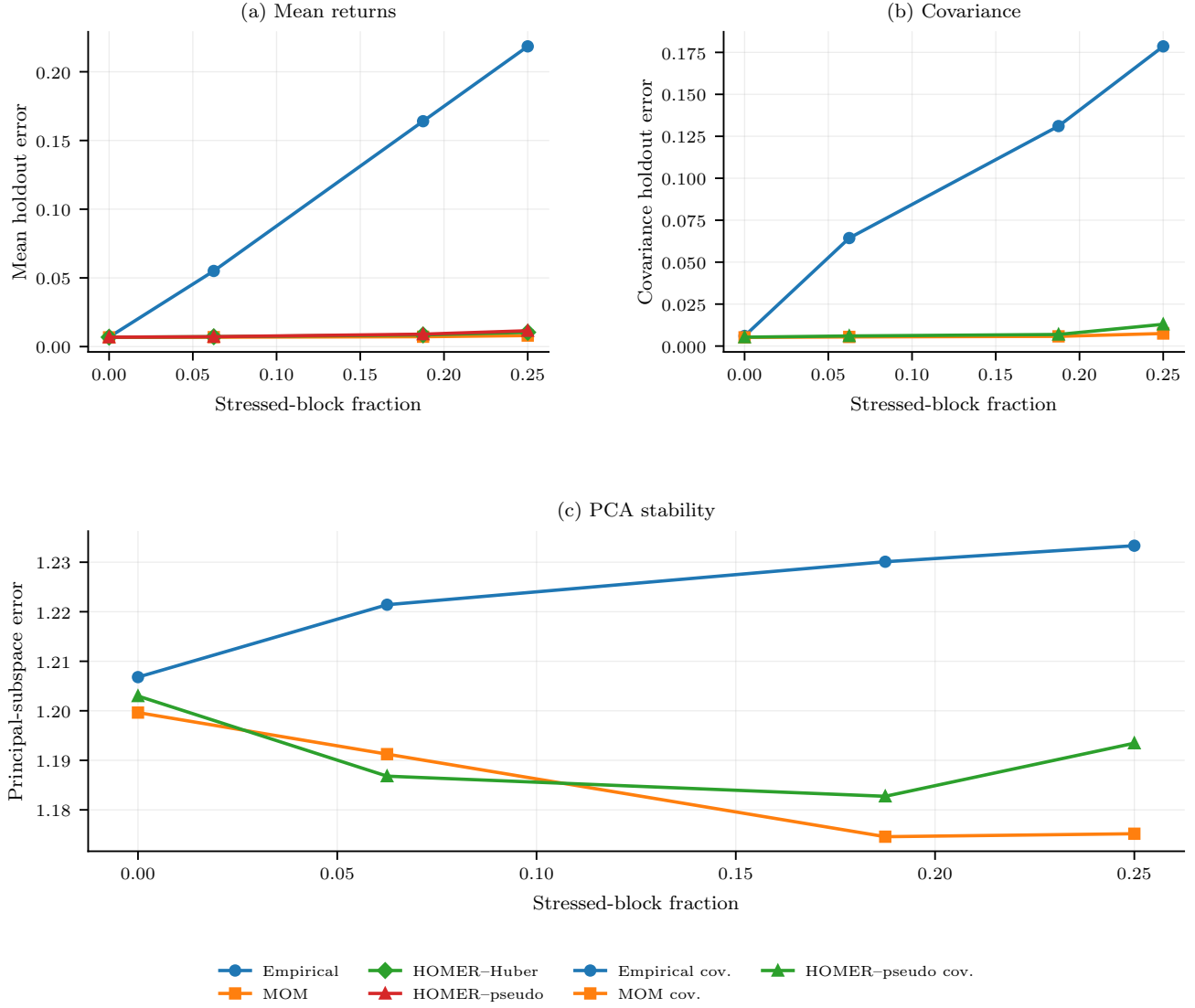}}
\caption{\textbf{Industry-return stability under artificial crash shifts.}
Mean holdout error, covariance holdout error, and rank-three projector error averaged over 31 overlapping rolling windows. The horizontal axis is the realized fraction of the 16 original training blocks selected for stress. These are descriptive dependent-data comparisons.}
\label{fig:industry-returns}
\end{figure*}

\begin{figure*}[p]
\centering
\resizebox{0.99\textwidth}{!}{\input{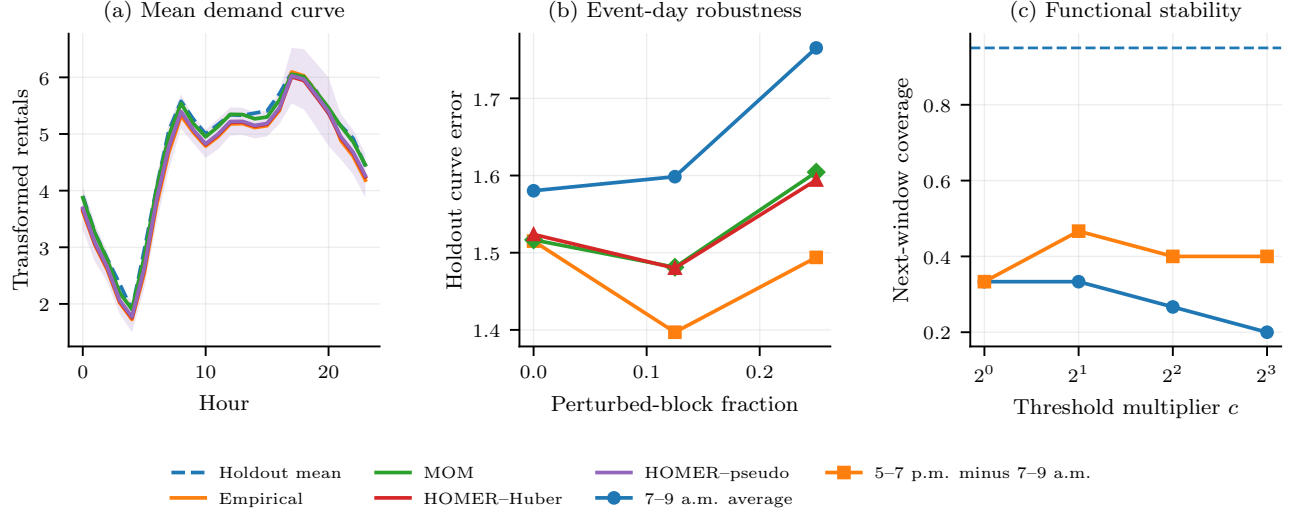}}
\caption{\textbf{Bike-demand curves.}
(a) Representative overall-demand fit with one of eight event-spike blocks. (b) Holdout curve error over 15 overlapping windows. The horizontal axis is the realized block fraction. (c) Descriptive unstressed next-window coverage for the 7 to 9 a.m. average and the 5 to 7 p.m. minus 7 to 9 a.m. contrast.}
\label{fig:bike-demand}
\end{figure*}

\begin{figure*}[p]
\centering
\resizebox{0.99\textwidth}{!}{\input{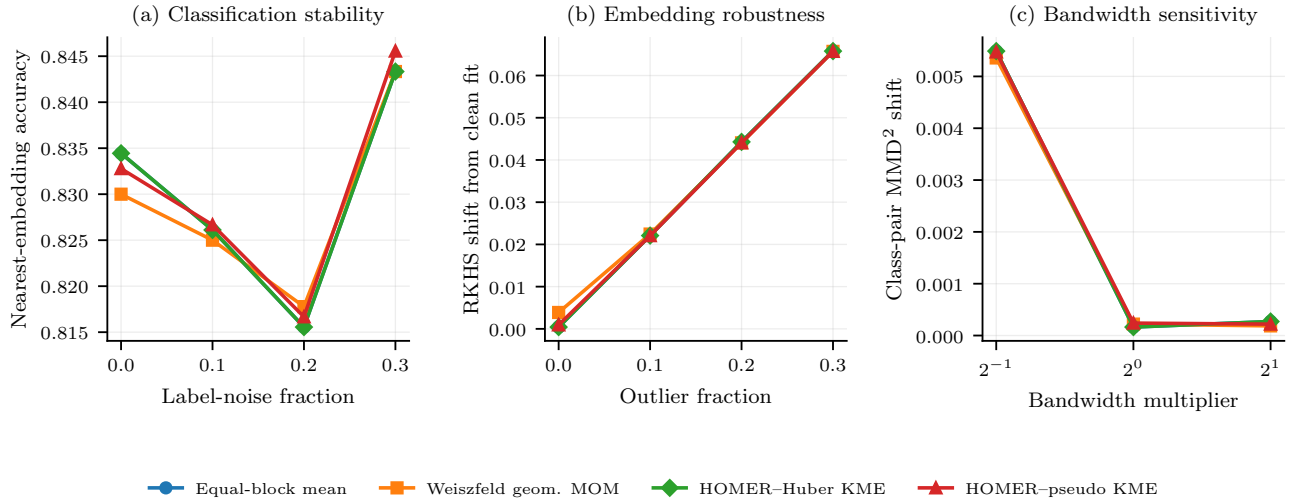}}
\caption{\textbf{Class-conditional HAR kernel embeddings.}
(a) Held-out nearest-embedding accuracy under label noise. (b) Mean clean-to-stressed RKHS displacement under outlier contamination. (c) Absolute shift in squared MMD between WALKING and WALKING\_UPSTAIRS. Panels use $k=16$ and $c=2$. The legend distinguishes the equal-block mean from the Weiszfeld geometric-MOM approximation.}
\label{fig:class-conditional-kernel}
\end{figure*}

\end{document}